\documentclass[11pt]{article}

\usepackage[left=1in,top=1in,right=1in,bottom=1in,letterpaper]{geometry}

\usepackage{amsmath,amssymb,amsfonts,amsthm,color,turnstile,mathrsfs,bm}
\usepackage{dsfont}
\usepackage{booktabs}
\usepackage{tabularx}
\usepackage{multirow}
\usepackage{subcaption}
\usepackage{wrapfig}
\usepackage{enumitem}
\usepackage{comment}
\usepackage{setspace}

\usepackage{graphicx}
\usepackage{pgffor}
\usepackage{epstopdf}
\usepackage{mathtools}
\usepackage{soul}
\usepackage[numbers,sort&compress]{natbib}
\usepackage{url}
\usepackage{hyperref}
\usepackage[capitalize,noabbrev]{cleveref}
\usepackage{float}

\usepackage[labelfont=bf,format=plain,justification=raggedright,singlelinecheck=false]{caption}
\usepackage[symbol]{footmisc}

\usepackage[ruled,vlined]{algorithm2e}

\allowdisplaybreaks

\theoremstyle{plain}
\newtheorem{thm}{Theorem}
\newtheorem{prop}{Proposition}
\newtheorem{lemma}{Lemma}

\theoremstyle{definition}
\newtheorem{defn}{Definition}
\newtheorem{assumption}{Assumption}

\theoremstyle{remark}
\newtheorem{rem}{Remark}

\crefname{lemma}{Lemma}{Lemmas}
\crefname{thm}{Theorem}{Theorems}
\crefname{prop}{Proposition}{Propositions}
\crefname{cor}{Corollary}{Corollaries}
\crefname{defn}{Definition}{Definitions}
\crefname{assumption}{Assumption}{Assumptions}
\crefname{rem}{Remark}{Remarks}
\crefname{equation}{Equation}{Equations}

\DeclareMathOperator*{\argmin}{arg\,min}
\renewcommand{\L}{\mathcal{L}}
\renewcommand{\epsilon}{\varepsilon}
\renewcommand{\phi}{\varphi}

\DeclareMathOperator{\N}{\mathbb{N}}
\DeclareMathOperator{\R}{\mathbb{R}}

\DeclareMathOperator{\E}{\mathbb{E}}
\DeclareMathOperator{\D}{\mathcal{D}}

\DeclareMathOperator{\G}{\mathcal{G}}
\DeclareMathOperator{\RR}{\mathcal{R}}
\DeclareMathOperator{\NN}{\mathcal{N}}

\title{Understanding In-Context Learning for Nonlinear Regression with Transformers: Attention as Featurizer}

\author{
Alexander Hsu\thanks{A. Hsu (hsu297@purdue.edu) is with the Department of Mathematics, Purdue University.} \quad
Zhaiming Shen\thanks{Z. Shen (zshen49@gatech.edu) is with the School of Mathematics, Georgia Institute of Technology.} \quad
Wenjing Liao\thanks{W. Liao (wliao60@gatech.edu) is with the School of Mathematics, Georgia Institute of Technology.} \quad
Rongjie Lai\thanks{R. Lai (lairj@purdue.edu) is with the Department of Mathematics, Purdue University.}
}

\date{} 

\begin{document}
\maketitle

\begin{abstract}
Pre-trained transformers are able to learn from examples provided as part of the prompt without any weight updates, a remarkable ability known as in-context learning (ICL). Despite its demonstrated efficacy across various domains, the theoretical understanding of ICL is still developing. Whereas most existing theory has focused on linear models, we study ICL in the nonlinear regression setting. Through the interaction mechanism in attention, we explicitly construct transformer networks to realize nonlinear features, such as polynomial or spline bases, which span a wide class of functions. Based on this construction, we establish a framework to analyze end-to-end in-context nonlinear regression with the constructed features. Our theory provides finite-sample generalization error bounds in terms of context length and training set size. We numerically validate the theory on synthetic regression tasks. 
\end{abstract}

\section{Introduction}

The ubiquity of Large Language Models (LLMs) is underpinned by the transformer architecture's \citep{Vaswani} distinct capacity for \emph{in-context learning} (ICL). Pre-trained transformers are able to learn a task from examples, known as \emph{context}, then perform the task on a \emph{query}, both of which are provided in the same prompt \citep{Radford,Brown,Garg}. Traditional feedforward models without self-attention mechanisms only use examples in pre-training; in contrast, ICL happens entirely at inference with no further weight updates. Despite its practical successes, a comprehensive theoretical framework explaining the mechanisms behind ICL is still emerging. 

Due to the complexity of directly analyzing large-scale models trained on vastly diverse data, theoretical analysis has shifted toward simplified settings, most notably linear regression via linear attention. The prevailing theoretical framework views ICL through the lens of implicit-optimization \citep{Hubinger,Dai}; for example, it has been shown that a single linear attention layer can reproduce one step of gradient descent (GD) on a given least squares objective in-context \citep{VonOswald, Mahankali} or directly solve least squares problems \citep{Akyurek, Zhang, Cole25}. It has also been shown that ReLU transformers can select and apply different algorithms in-context \citep{Bai}. A complementary line of work studies the training dynamics of ICL for linear models \citep{Ahn,Mahankali,Raventos,Huang,Lu,Zhang2025}.

Recent work has begun to extend analysis beyond the linear setting. Some works have studied ICL through the lens of kernel methods \citep{Cheng,Shen2}. In \citet{Cheng}, transformers equipped with a kernel activation function are shown to perform functional gradient descent in function space with respect to the Reproducing Kernel Hilbert Space (RKHS) metric induced by the kernel. In \citet{Shen2}, transformers are demonstrated to perform kernel regression, and a generalization error analysis for in-context regression on manifolds is established. An alternative approach employs feedforward neural networks (FFNs) to construct nonlinear feature maps or basis functions, after which a transformer operates on the learned feature representation \citep{Suzuki24,Guo,Li,Sun}. A numerical study about the ICL of polynomials \citep{Naim} shows that transformers can predict polynomial functions in context if the testing distribution is the same as the training distribution. 

Our work studies transformer-based in-context learning (ICL) for nonlinear regression with polynomial or spline features. Unlike previous modular approaches which use separate architectural components for feature representation, we shift the approximation power to the attention mechanism itself, building an end-to-end transformer with ReLU-activated attention for the entire pipeline of nonlinear regression with fixed features. This involves leveraging the interaction mechanism in attention to perform interpretable in-context arithmetic operations. We analyze this construction to derive generalization error bounds for the ICL of a wide function class which can be approximated by polynomials or splines. We connect the theory to classical regression schemes, such as polynomial and spline regression.

We summarize our contributions as follows.
\begin{enumerate}[topsep=2pt, itemsep=2pt, parsep=0pt, partopsep=0pt]
    \item We construct shallow and wide transformer networks to realize nonlinear regression with polynomial or spline features. The prompt is featurized with interpretable attention weights performing basic arithmetic operations in-context, and a final linear attention layer approximates the solution to the resulting least squares problem. 
    
    \item We derive complete generalization error bounds in terms of context length and training set size, which follow a bias-variance decomposition into approximation and statistical errors. 

    \item We numerically verify our theoretical bounds with synthetic experiments, and compare performance with entirely linear and softmax transformers. We also ablate different components.
\end{enumerate}

\subsection{Related Works}
Closest to our work is \citet{Suzuki24}, which recovers minimax optimal bounds for in-context nonlinear regression of Besov functions with a deep feedforward network followed by a linear attention layer. Whereas they rely on existence results for universal approximation by deep feedforward networks, our primary novelty is the explicit construction of features through attention. As a consequence, our transformer networks are shallow, with depth independent of desired accuracy, and we are able to construct the featurization (polynomials and splines) without any error. For example, we show a two-layer \emph{attention-only} transformer that can provably implement regression with a linear cardinal B-spline basis. Our approximation proofs are also fully constructive. 

Outside of the ICL setting, single task learning with transformers has been studied extensively \citep{Yun,Takakura,Havrilla24,Shen}. Our work partially incorporates this aspect, where the task learned is the chosen feature map. 

There is a significant amount of interest in \emph{interpreting} the inner workings of pre-trained transformers \citep{Elhage,Rai}, in particular on arithmetic problems \citep{Jelassi,Stolfo,Quirke,McLeish,Yang}. Our explicit construction of the featurization gives an example of a transformer with interpretable weights that perform in-context arithmetic; the framework may be of independent interest.

\textbf{Notation.} Throughout, we denote vectors by bold lowercase letters, scalars by plain lowercase letters, and matrices by plain uppercase letters. $\|\mathbf{v}\|$ denotes the $\ell^2$ norm of the vector $\mathbf{v}$. We also use $\|A\|$ to denote the spectral norm for the matrix $A$, and $\|A\|_{\max}$ to denote the max norm, which returns the largest entry in absolute value of the matrix $A$. We will use $\theta$ as notation for a set of weight matrices; for any such collection, we define the norm $\|\theta\|_\infty=\max_{W\in\theta}\|W\|_{\max}$ as the largest magnitude among all weight parameters. We denote by $\mathbf{e}_i$ the $i$-th standard basis vector $(0,\ldots, 0, 1, 0, \ldots, 0)^T$, where the dimension is clear from the context.

\section{Problem Setup}\label{sec:setup}

\subsection{In-Context Regression}

In the ICL setting, a single transformer is trained to regress on any function $f$ drawn from a broad function class, based solely on examples provided as context. We restrict our discussion to the regression of \emph{real} functions for the majority of this work; in \cref{sec:highdextensions} we extend to vector-valued functions and multivariable functions.

\vspace{-0.2cm}
\paragraph{Data Generation}
Let $X \subset \R$ be the input domain, and let $\D_F$ be a distribution over functions $f \colon  X \to \R$. In the ICL setting, each task is generated by first sampling a ground-truth function $f \sim \D_F$. Rather than observing $f$ directly, the learner is provided with a finite \emph{context} consisting of input--output pairs
$\{(x_i,y_i)\mid y_i=f(x_i)\}_{i=1}^n$
where the inputs $x_i$ are drawn i.i.d.\ from a distribution $\D_X$ on $X$, independently of $f$, and   $n$ is referred to as the \emph{context length}.

The context serves as an implicit description of the underlying task, from which the learner must infer the behavior of $f$ solely by conditioning on the context, without updating model parameters. More precisely, we are interested in an operator $G$ such that for any given prompt/task 
\[\mathbf{s}^f\coloneqq (x_1,y_1,x_2,y_2,\ldots,x_n,y_n, x) \text{ with } y_i = f(x_i),\]  the operator output $G(\mathbf{s}^f)$ predicts $y\coloneqq f(x)$ on a new \emph{query} $x\sim\D_X$. Our focus is to study how well $G$ can be approximated by a transformer network $G_\theta$.

\paragraph{Training, Risk, and Generalization}
Our transformer has access to a training set $\Gamma=\{(\mathbf{s}^{f_\ell},y^\ell)\}_{\ell=1}^L$ of $L$ prompts and their corresponding true solutions, where each training datum has the form $(\mathbf{s}^{f_\ell},y^\ell)=(x_1^\ell,y_1^\ell,x_2^\ell,y_2^\ell,\ldots,x_n^\ell,y_n^\ell,x^\ell,y^\ell=f_\ell(x^\ell))$. We assume i.i.d. distributions for training and testing: $f_\ell\sim\D_F$ and $x_i^\ell\sim\D_X$. 

Let $\G$ denote a transformer network class, defined in \cref{def:transformernetwork}, consisting of networks $G_\theta\in \G$ with parameters $\theta$. We consider a network trained according to empirical risk minimization (ERM) over the training set: 
\begin{equation}\label{eq:erm}
G_{\hat\theta}^\Gamma = \argmin_{G_\theta\in\G} \RR^\Gamma(G_\theta)
\end{equation}
where $\RR^\Gamma(G_\theta)\coloneqq \frac{1}{L}\sum_{\ell=1}^L|G_\theta(\mathbf{s}^{f_\ell})-y^\ell|^2$ and $\hat{\theta}$ denotes the optimal parameters. The superscript $\Gamma$ emphasizes the dependence on the training data set $\Gamma$.
We also define the analogous population risk $\RR(G_\theta)\coloneqq \E_{\mathbf{s}^f}|G_\theta(\mathbf{s}^f)-y|^2$, where $\E_{\mathbf{s}^f}$ is compact notation for $\E_{x_1,\ldots,x_n,x\sim\D_X, f\sim\D_F}$.

Our main goal is to study the \emph{generalization error} of the empirical risk minimizer $G_{\hat\theta}^\Gamma$ on a new sample $\mathbf{s}^f$: 
\[
    \RR(G^\Gamma_{\hat{\theta}}(\mathbf{s}^f))\coloneqq| G^\Gamma_{\hat{\theta}}(\mathbf{s}^f)-f(x) |^2.
\]

Taking expectation with respect to training data and the test sample, the generalization error can be characterized as
\[
    \E_{\Gamma}\E_{\mathbf{s}^f}[\RR(G_{\hat{\theta}}^\Gamma(\mathbf{s}^f))] = \E_{\Gamma}[\RR(G_{\hat\theta}^\Gamma)].
\]

\subsection{Transformer Architecture}
We define the transformer architecture used in this paper.

\begin{defn}[(Multi-head) Attention]
    An \emph{attention head} is defined as 
    \begin{equation}\label{eq:attn}\mathrm{A}_{Q,K,V}(H)=VH\sigma((KH)^TQH)\end{equation}
    where $Q,K,V\in \R^{d_{\mathrm{embed}}\times d_{\mathrm{embed}}}$ are the query, key, and value matrices respectively, and $\sigma$ is an activation function. We use \emph{ReLU activation} for all but the last attention layer, which instead uses \emph{linear attention} with $\sigma=\mathsf{Id}$.

   \emph{Multi-head attention} (MHA) with $m$ heads sums over multiple instances of attention performed in parallel:
    \begin{equation}\label{eq:MHA}\mathrm{MHA}(H)=\sum_{j=1}^mV_jH\sigma((K_jH)^TQ_jH).\end{equation}

    Note that this summation form of MHA is a special case of the more standard MHA with concatenation \citep{Vaswani}, by choosing the projections to reproduce the sum. This variant is often preferred in theory for ease of analysis. 
\end{defn}

\begin{defn}[Feedforward Network (FFN) Class]
    The \emph{class of feedforward networks} is denoted $\mathcal{FFN}(L_{\mathrm{FFN}},w_{\mathrm{FFN}})$, and consists of FFNs with at most $L_{\mathrm{FFN}}$ layers, each of width at most $w_{\mathrm{FFN}}$. FFNs use ReLU activation and act columnwise.
\end{defn}

Transformer blocks are defined as residual combinations of the two aforementioned maps.

\begin{defn}[Transformer Blocks]\label{def:transformerblock}
    A \emph{transformer block} $\mathrm{B}$ is given by
    \[
    \mathrm{B}(H)=\mathrm{FFN}(\mathrm{MHA}(H)+H)+\mathrm{MHA}(H)+H.
    \]

    The \emph{class of transformer blocks} $\mathcal{B}(m,L_{\mathrm{FFN}},w_{\mathrm{FFN}})$ consists of transformer blocks with $m$-headed attention and a feedforward component in $\mathcal{FFN}(L_{\mathrm{FFN}},w_{\mathrm{FFN}})$.

\end{defn}

\begin{defn}[Transformer Networks]\label{def:transformernetwork}
    A \emph{transformer network} $G_\theta$ with parameters $\theta$ is the composition of an embedding layer plus positional encoding module, a sequence of transformer blocks, and a decoding layer, i.e.
    \[
        G_\theta=\mathrm{D}\circ\mathrm{B}_{L_G}\circ\cdots\circ\mathrm{B_1}\circ(\mathrm{PE}+\mathrm{E}).
    \]

    Here $L_G$ is the number of transformer blocks (also called the \emph{depth} of the transformer); the blocks are determined by the parameters $\theta$. The embedding $\mathrm{E}$, positional encoding $\mathrm{PE}$, and decoding $\mathrm{D}$ are not trained, with specifics depending on the approach used. In general, $\mathrm{PE}+\mathrm{E}(\mathbf{s}^f)$ embeds the prompt into an $\R^{d_\mathrm{embed}\times \ell}$ matrix of the form
    \begin{equation}\label{eq:encoding}
    \mathrm{PE}+\mathrm{E}(\mathbf{s}^f) =
    \begin{bmatrix}
    x_1 & x_2 & x_3 & \cdots & x_n & x \\
    0 & 0 & 0 & \cdots & 0 & 0 \\
    \vdots & \vdots & \vdots & \ddots & \vdots & \vdots \\
    0 & 0 & 0 & \cdots & 0 & 0 \\
    y_1 & y_2 & y_3 & \cdots & y_n & 0 \\
    0 & 0 & 0 & \cdots & 0 & 0 \\
    \mathcal{I}_1 & \mathcal{I}_2 & \mathcal{I}_3 & \cdots & \mathcal{I}_n & \mathcal{I}_{n+1} \\
    1 & 1 & 1 & \cdots & 1 & 1
    \end{bmatrix}
    \end{equation}
    with sinusoidal positional encoding terms $\mathcal{I}_i\coloneqq (\cos(\frac{i\pi}{2\ell}),\sin(\frac{i\pi}{2\ell}))^T$. We have sequence length $\ell=n+1$ in this case. The embedding dimension $d_\mathrm{embed}$ depends on the features we use, with additional rows of zeros in the middle for computation. The decoder is fixed to output the entry to the right of $y_n$, indexed by $(d_\mathrm{embed}-4,n+1)$, where the regression result is placed.

    Finally, we define the \emph{class of transformer networks} 
    \[\G(L_G,m_G,d_{\mathrm{embed}},\ell, L_{\mathrm{FFN}}, w_{\mathrm{FFN}},R_G,\kappa)\] as consisting of networks $G_\theta$ with $L_G$ transformer blocks from $\mathcal{B}(m_G,L_{\mathrm{FFN}},w_{\mathrm{FFN}})$, embedding dimension (matrix height) $d_{\mathrm{embed}}$, sequence length (matrix width) $\ell$, output bound $\|G_\theta(\cdot)\|_{L^\infty(X)}\leq R_G$, and weight bound $\|\theta\|_\infty\leq\kappa$.

\end{defn}

\section{In-Context Polynomial Regression}\label{sec:polyregression}

\subsection{Regression with Features}
The method of \emph{feature expansion} linearizes nonlinear regression problems by first transforming the input data through specified nonlinear feature maps, allowing for a linear fit in the resulting feature space. The resulting features form the basis for the function space that our model is able to approximate.

Let $f$ be the function we are interested in approximating. For a predetermined set of features $\{\phi_j\colon X\to\R\}_{j=1}^w$, we define the \emph{feature map} $\phi\colon X\to\R^w$ which transforms an input $x$ into a vector of features:
\[\phi(x)\coloneqq[\phi_1(x),\phi_2(x),\ldots,\phi_w(x)]^T.\]

The approximation of $f$ has the form
\begin{equation}\label{eq:fpolyapprox}
    f(x)\approx\sum_{i=1}^wa_i\phi_i(x)=\mathbf{a}^T\phi(x),
\end{equation}
where $\mathbf{a}=[a_1,\ldots, a_w]^T$ is a vector of coefficients. The prediction at a new query point $x$ can be found by computing $\mathbf{a}^T\phi(x)$. 

The coefficients $\mathbf{a}$ are often determined by ordinary least squares (OLS), minimizing the sum of squared residuals between predicted and observed values. For our observed context $\{(x_i,y_i)\mid y_i=f(x_i)\}_{i=1}^n$, the resulting best fit satisfies the matrix equation $\mathbf{y}^T=\mathbf{a}^T\Phi$, which expands as
\[\begin{bmatrix}
    y_1 \\ y_2 \\ \vdots \\ y_n
\end{bmatrix}^T=\begin{bmatrix}
    a_1 \\ a_2 \\ \vdots \\ a_w
\end{bmatrix}^T 
\begin{bmatrix}
    \phi_1(x_1) & \phi_1(x_2) & \cdots  & \phi_1(x_n) \\
    \phi_2(x_1) & \phi_2(x_2) & \cdots & \phi_2(x_n) \\
    \vdots & \vdots & \ddots & \vdots \\
    \phi_w(x_1) & \phi_w(x_2) & \cdots & \phi_w(x_n)
\end{bmatrix}.\]
We refer to the $w\times n$ matrix $\Phi$ as the \emph{feature matrix}, where the $i$-th column corresponds to the feature vector $\phi(x_i)$ of the $i$-th input $x_i$. 

In \cref{sec:icpolyapprox} and \cref{sec:generror}, we focus our analysis on polynomial regression with monomial features. For polynomial regression by degree-$d$ polynomials, the feature map $\phi\colon X\to\R^{d+1}$ is defined as
\[\phi(x)= \mathbf{v}(x)\coloneqq [1,x,x^2,\ldots,x^d]^T\]
and the feature matrix $\Phi$ is known as a Vandermonde matrix. Although polynomial regression is a simple stylized model, it provides a canonical and concrete framework for nonlinear regression via featurization. Recent work also supports polynomials as an interesting function class for understanding ICL in general \citep{Wilcoxson}. For these reasons, we present a rigorous analysis of in-context polynomial regression and derive explicit generalization error bounds. 

We extend to other choices of featurization, including splines, in \cref{sec:splinesandbeyond}. Regardless of the choice of features, the optimization problem remains solvable via OLS.

\subsection{Realizing Polynomial Regression In-Context}\label{sec:icpolyapprox}

We implement the full polynomial regression pipeline (including evaluation on a new point) entirely in-context with an end-to-end transformer.  We summarize the main ideas here, with more details in \cref{sec:approxerrorproof}. The transformer has three important components discussed below.

\paragraph{Embedding.} The prompt is embedded in a matrix of the form specified in \eqref{eq:encoding}, where here $d_{\mathrm{embed}}=d+7$ and $\ell=n+1$. 

\paragraph{Feature Representation.} The early transformer blocks form the feature matrix (Vandermonde) in-context, placing it above the row of outputs $y_i$. We rely on the \emph{Interaction Lemma} \labelcref{lem:interaction}, which cleverly leverages ReLU-activated attention and sinusoidal positional encoding to constructively isolate interactions between specific entries. First introduced in \citet{Havrilla24}, later work demonstrated that many basic arithmetic operations could be implemented within a transformer via the Interaction Lemma as sparse interpretable updates \citep[Table 1]{Shen}.

We build the whole matrix recursively and without error, performing many operations in parallel with MHA. After these transformer blocks, we are left with the following size $\R^{(d+7) \times (n+1)}$ matrix:

\begin{equation}\label{eq:matrixbeforefinal}
    Z \coloneqq \begin{bmatrix}
    x_1 & x_2 & x_3 & \cdots & x_n & x \\
    1 & 1 & 1 & \cdots & 1 & 1 \\
    x_1 & x_2 & x_3 & \cdots & x_n & x \\
    x_1^2 & x_2^2 & x_3^2 & \cdots & x_n^2 & x^2 \\
    \vdots &  &  & \ddots & & \vdots \\
    x_1^d & x_2^d & x_3^d & \cdots & x_n^d & x^d \\
    y_1 & y_2 & y_3 & \cdots & y_n & 0 \\
    0 & 0 & 0 & \cdots & 0 & 0 \\
    \mathcal{I}_1 & \mathcal{I}_2 & \mathcal{I}_3 & \cdots & \mathcal{I}_n & \mathcal{I}_{n+1} \\
    1 & 1 & 1 & \cdots & 1 & 1
\end{bmatrix}.
\end{equation}

\paragraph{Solving Least Squares.} A final \emph{linear} transformer block and decoder directly outputs the solution to the OLS problem evaluated on the query. The final transformer block uses a linear attention layer ($\sigma=\mathrm{Id}$) with one head, in line with the literature \citep{Zhang,Cole25}. We follow the common notation for linear attention, which merges two parameters from \eqref{eq:attn}. On a matrix $H\in \R^{d_{\mathrm{embed}}\times \ell}$, a linear attention head $A^\mathrm{lin}_{Q,V}(H)$ with parameters $Q,V\in \R^{d_{\mathrm{embed}}\times d_{\mathrm{embed}}}$ is defined as
\[
    \mathrm{A}^\mathrm{lin}_{Q,V}(H)=VH\frac{H^TQH}{\rho(\ell)},
\]
where the normalization factor $\rho(\ell)\in\N$ is a function of sequence length $\ell$. Notice that this is a special case of the previously defined attention head in \eqref{eq:attn}: with $\sigma=\mathrm{Id}$, the output becomes $VHH^TK^TQH$, and $K^TQ$ are merged into a single parameter $Q$ (we may also incorporate the normalization constant here). We extend this to a transformer block with trivial FFN component and one attention head: 
\[\mathrm{B}^\mathrm{lin}_{Q,V}(H)=\mathrm{A}^\mathrm{lin}_{Q,V}(H)+H.\]

An adaptation of a well-studied parameterization lets us approximate the least squares solution. Recall that the decoder reads out the $(d+3,n+1)$-th coordinate of the matrix. Choosing normalization $\rho(\ell)=\ell-1$, with parameters of the form
\[
V=\begin{bmatrix}
    0_{(d+2)\times (d+2)} & 0_{(d+2)\times 1} & 0_{(d+2)\times 4} \\
    0_{1\times (d+2)} & p & 0_{1\times 4} \\
    0_{4\times (d+2)} & 0_{4\times 1} & 0_{4\times 4}
\end{bmatrix}\qquad\text{ and }\qquad
    Q=\begin{bmatrix}
    0_{1\times 1} & 0_{1\times (d+1)} & 0_{1\times 5} \\
    0_{(d+1)\times 1} & \tilde{Q} & 0_{(d+1)\times 5} \\
    0_{5\times 1} & 0_{5\times (d+1)} & 0_{5\times 5}
\end{bmatrix}
\]
where $p\in \R$ and $\tilde{Q}\in\R^{(d+1)\times (d+1)}$, a simple derivation shows that the final output after decoding is given by 
\begin{equation}\label{eq:linearoutput}
    D(\mathrm{B}^\mathrm{lin}_{Q,V}(Z))=p\left(\frac{1}{n}\sum_{i=1}^ny_i\mathbf{v}(x_i)^T\right)\tilde{Q}\mathbf{v}(x).
\end{equation} 

By \eqref{eq:fpolyapprox}, $y_i\approx \mathbf{a}^T\mathbf{v}(x_i)$ for all $1\leq i\leq n$ (this approximation will be made precise in \cref{ass:polyapproximability}). If we substitute that in and further let $p=1$ and $\tilde{Q}=\Sigma^{-1}$ where $\Sigma$ is the \emph{uncentered covariance} $\Sigma\coloneqq \E_{x\sim\D_X}[\mathbf{v}(x)\mathbf{v}(x)^T]$, \eqref{eq:linearoutput} simplifies to approximate $\mathbf{a}^T\Sigma_n\Sigma^{-1}\mathbf{v}(x)$ where $\Sigma_n$ is the \emph{empirical} uncentered covariance $\Sigma_n\coloneqq \frac{1}{n}\sum_{i=1}^n\mathbf{v}(x_i)\mathbf{v}(x_i)^T$. As $n\to\infty$, $\Sigma_n\to\Sigma$ and the output approaches the polynomial regression approximation for $f(x)$. 

The choice of $V$ and $Q$ here are essentially the standard choice for the ICL of linear systems \citep{Zhang,Cole25}, modified with zero padding to mask away the additional rows unnecessary for the final step.

\section{Generalization Error for In-Context Polynomial Regression}\label{sec:generror}

\subsection{Generalization Error Analysis}

We explicitly restrict our function space to a tubular neighborhood of a subset of polynomials. The generality of the function space depends on the chosen polynomial degree $d$: the Stone-Weierstrass theorem states that every continuous function on a closed interval is approximable by a polynomial of sufficiently high degree. 

\begin{assumption}[Polynomial Approximability of Function Class]\label{ass:polyapproximability}
For a given degree $d$ and constants $\delta,\alpha, R_F\geq 0$ and a domain $X$, 
the distribution $\D_F$ is supported on the function class 
\[
    F(\delta,\alpha,d,R_F)\coloneqq  \bigg\{f \mid \min_{p\in\Pi_d^\alpha}\|f-p\|_{L^\infty(X)}\leq \delta\text{ and }\|f\|_{L^\infty(X)}\leq R_F\bigg\}
\]
where $\Pi_d^\alpha$ consists of polynomials of the form $p(x)=\sum_{i=0}^da_ix^i$ with coefficients satisfying \newline $\|(a_0,\ldots,a_d)\|\leq\alpha$. The bound $R_F$ can be expressed in terms of $\alpha$ and the bounded domain specified below.
\end{assumption}

Next, we make assumptions on $\D_X$. 

\begin{assumption}\label{ass:conditioning}
$\D_X$ is supported on $X\subseteq[-R,R]$ and the covariance matrix $\Sigma\coloneqq \E_{x\sim\D_X}[\mathbf{v}(x)\mathbf{v}(x)^T]$ has eigenvalues bounded away from $0$. Namely, there is $\tau>0$ such that 
\[\|\Sigma^{-1}\|=\frac{1}{\lambda_{\min}(\Sigma)}\leq \tau\]
\end{assumption}

This is closely related to the conditioning of the feature matrix. We note that related works make similar assumptions to the same effect, such as orthogonality of basis vectors \citep{Suzuki24}.

We denote $R_{\mathbf{v}}\coloneqq\sup_{x\in X}\|\mathbf{v}(x)\|\leq\sqrt{d+1}\max(R^d,1)$. The completed matrix in \eqref{eq:matrixbeforefinal} has norm bound $U\coloneqq\|Z\|_{\max}=\max(R_{\mathbf{v}}, R_F)$.

 We now state our main result below; the meaning of each constant and variable is summarized in \cref{tab:notation}.

\begin{thm}[Generalization Error]\label{thm:generror}
Suppose $\D_F$ satisfies \cref{ass:polyapproximability} and $\D_X$ satisfies \cref{ass:conditioning}. Fix the parameters of the transformer network class $\G(L_G,m_G,d_{\mathrm{embed}},\ell, L_{\mathrm{FFN}}, w_{\mathrm{FFN}},R_G,\kappa)$ as 
\begin{align*}
    & L_G = O(\log d), m_G=O(dn), d_{\mathrm{embed}}=d+7, \ell=n+1 \\
    & L_{\mathrm{FFN}}=O(1), w_{\mathrm{FFN}}=d+7, \kappa = O(d^4R_\mathbf{v}^2n^2U^2+\tau).
\end{align*}
where $O(\cdot)$ hides constants.
The empirical risk minimizer $G^\Gamma_{\hat{\theta}}$ given in \eqref{eq:erm} satisfies 
        \[
         \E_\Gamma[\RR(G_{\hat\theta}^\Gamma)] \leq C_1\bigg(\frac{\alpha^2 \tau^2 R_\mathbf{v}^{6}\log d}{n}+\delta^2 R_\mathbf{v}^4\tau^2+\delta^2\bigg) + C_2\bigg(\frac{\sqrt{nd^3(\log d)^3\log((d^4R_\mathbf{v}^2n^2U^2+\tau)L)}}{\sqrt{L}}\bigg)\]
where $C_1$ and $C_2$ are absolute constants.
\end{thm}

\begin{rem}
    In terms of $n$ and $L$, Theorem \ref{thm:generror} gives rise to the following rate
    \[
        \E_\Gamma[\RR(G_{\hat\theta}^\Gamma)] = O\left(   \frac{1}{n} + \frac{\sqrt{n\log(n^2L)}}{\sqrt{L}}\right).
    \]
    
    The first term, corresponding to the approximation error, recovers the Monte Carlo convergence rate for the mean squared regression error on $n$ points. The statistical error in the second term decreases as the training sample size $L$ increases. The statistical error 
    has a dependence on $\sqrt{n}$ in our framework because our construction utilizes $n$ attention heads, thus a larger context length demands a larger network to process the context. This differs from the linear setting, where a single linear attention head can process contexts of any length. If the training size $L$ is large enough, the statistical error in the second term is guaranteed to be small.
\end{rem}

\subsection{Proof Sketch of \cref{thm:generror}}

We decompose the error $\RR(G_{\hat\theta}^\Gamma)$ as
\begin{equation}\label{eq:generrordecomp}
\begin{aligned}
\RR(G_{\hat\theta}^\Gamma) &=
\big[\RR(G_{\hat{\theta}}^\Gamma)-\RR^\Gamma(G^\Gamma_{\hat{\theta}})\big]
+\big[\RR^\Gamma(G_{\hat\theta}^\Gamma)-\RR^\Gamma\left(G_{\tilde{\theta}}\right)\big] +\big[\RR^\Gamma\left(G_{\tilde{\theta}}\right)-\RR\left(G_{\tilde{\theta}}\right)\big]
+\RR\left(G_{\tilde{\theta}}\right) \\
&\leq 2 \sup_{G_\theta \in \G}|\RR(G_\theta)-\RR^\Gamma(G_\theta)|
+ \RR\left(G_{\tilde{\theta}}\right)
\end{aligned}
\end{equation}
where $G_{\tilde{\theta}}$ is a network we explicitly construct from the network class $\G$, and we discarded the second term $\big[\RR^\Gamma(G_{\hat{\theta}}^\Gamma)-\RR^\Gamma\left(G_{\tilde{\theta}}\right)\big]$ since $\hat\theta$ is a minimizer. Taking expectation, the generalization error $\E_\Gamma[\RR(G_{\hat\theta}^\Gamma)]$ is upper bounded by
\[
\underbrace{\RR\left(G_{\tilde{\theta}}\right)}_{\text{I: Approximation Error}} + 2 \underbrace{\E_\Gamma[\sup_{G_\theta\in \G}|\RR(G_\theta)-\RR^\Gamma(G_\theta)|]}_{\text{II: Statistical Error}}.
\]

Note that the approximation error is independent of training data, hence is without expectation. I is bounded by \cref{lem:approxerror} and II is bounded by \cref{lem:staterror}. The network hyperparameters are determined by the approximation construction, which in turn specify the network class for the statistical error part. Putting the bounds together yields the final result.

\paragraph{Approximation Theory}

\begin{lemma}[Approximation Error]\label{lem:approxerror}
Suppose \cref{ass:polyapproximability} and \cref{ass:conditioning} are satisfied. There exists a transformer network $G_{\tilde\theta}$ with hyperparameters 
\vspace{-0.1cm}
\[
\begin{aligned}
    & L_G = O(\log d), m_G=O(dn), d_{\mathrm{embed}}=d+7, \ell=n+1 \\
    & L_{\mathrm{FFN}}=O(1), w_{\mathrm{FFN}}=d+7, \kappa = O(d^4R_\mathbf{v}^2n^2U^2+\tau)
\end{aligned}
\]
and an absolute constant $C_1$, such that 
\[\RR(G_{\tilde{\theta}})\leq C_1\left(\frac{\alpha^2 \tau^2 R_\mathbf{v}^{6}\log d}{n}+\delta^2 R_\mathbf{v}^4\tau^2+\delta^2\right)
\] 
\end{lemma}

Our constructive proof enjoys two important properties:

\textbf{Universality.} In \cref{lem:approxerror}, both the architecture and parameters are \emph{universal} for all functions of the class $F$ specified in \cref{ass:polyapproximability}. 

\textbf{Effective Representation.} Using a shallow transformer, we are able to construct the feature representation without any error due to the simple multiplicative definition of a monomial basis being compatible with the Interaction Lemma \labelcref{lem:interaction}. In contrast, deep ReLU FFNs can only approximate multiplication with certain accuracy, and require depth to scale with desired accuracy \citep{Yarotsky}. 

\begin{proof}[Proof Sketch of Lemma \ref{lem:approxerror}]
The main ideas are outlined here, with the full proof in \cref{sec:approxerrorproof}. 

We first decompose the approximation error. By \cref{ass:polyapproximability}, there exists a polynomial $p_{f}(x)\in\Pi_d^\alpha$ with coefficients satisfying $\|\mathbf{a}\| \leq \alpha$  such that 
\[\|f-p_{f}\|_{L^\infty(X)}\leq \delta.\]

This yields the following decomposition:
\begin{align*}
    \RR(G_{\tilde{\theta}}) &= \E_{\mathbf{s}^f}|G_{\tilde{\theta}}(\mathbf{s}^f)-f(x)|^2 \\
    &= \E_{\mathbf{s}^f}|G_{\tilde{\theta}}(\mathbf{s}^f)-p_{f}(x)+p_{f}(x)-f(x)|^2 \\
    &\leq 2\E_{\mathbf{s}^f}|G_{\tilde{\theta}}(\mathbf{s}^f)-p_{f}(x)|^2 +2\E_{x\sim \D_X, f\sim \D_{F}} |p_{f}(x)-f(x)|^2 \\
    &\leq 2\E_{\mathbf{s}^f}|G_{\tilde{\theta}}(\mathbf{s}^f)-p_{f}(x)|^2 + 2\delta^2.
\end{align*}  

The first term is bounded by an explicit construction of an oracle transformer network, following the approach outlined in \cref{sec:icpolyapprox}. In particular, for a fixed set of monomial basis functions, we describe a transformer which estimates the solution to polynomial regression under that basis. 

We are able to realize the feature functions exactly, without incurring any error. We analyze the error of the final approximation with the last linear attention layer via matrix concentration inequalities, and also incorporate how the polynomial approximation error propagates through. \phantom{\qedhere}
\end{proof}

\paragraph{Statistical Theory}
We bound the statistical error term with standard techniques from learning theory, involving Rademacher complexity and an explicit covering number computation for our transformer network class. We state the statistical error here and defer the proof to \cref{sec:staterrorproof}.
\begin{lemma}[Statistical Error]\label{lem:staterror}
With $\G$ having hyperparameters as specified in \cref{lem:approxerror}, 
\[
    \E_\Gamma[\sup_{G_\theta\in \G}|\RR(G_\theta)-\RR^\Gamma(G_\theta)|] 
    \leq C_2\bigg(\frac{\sqrt{nd^3(\log d)^3\log((d^4R_\mathbf{v}^2n^2U^2+\tau)L)}}{\sqrt{L}}\bigg)
\]
for an absolute constant $C_2$.
\end{lemma}

\section{Splines and Beyond}\label{sec:splinesandbeyond}
In this section, we extend the theoretical framework established for polynomials to splines \citep{Devore}, along with studying high-dimensional settings. Importantly, the only change from polynomial regression is the choice of features, thus only a minor modification of the construction is required to apply the same complete analysis we rigorously did in \cref{sec:generror} here.

\subsection{Spline Methods}
A \emph{spline} is a piecewise polynomial function, with domain split by a set of \emph{knots} in $X$. After fixing a set of knots, spline regression fits a spline to data by fitting local low-degree polynomials, expressible in terms of a \emph{B-spline basis}. Spline regression is thus another example of linear regression with fixed features: we featurize the inputs with a B-spline basis, then solve the linear system with OLS. 

B-spline bases can be defined via the \emph{Cox-de Boor recursion formula}. Suppose the input domain is a closed interval $X=[a,b]\subseteq[-R,R]$ and we have $m+1$ distinct knots $t_1,t_2,\ldots, t_{m+1}\in X$ with $t_1=a$, $t_{m+1}=b$, and $t_j< t_{j+1}$. This breaks the interval up into $m$ bins. We restrict to the case of equally spaced knots along $X$, where we can use \emph{cardinal B-splines} that have a simple definition using ReLU. The $q$-th degree cardinal B-spline $B_q$ (which will fit degree $q$ local polynomials) can be expressed by the following formula:
\begin{equation}\label{eq:cardinalbspline}
B_q(x) = \frac{1}{q!}\sum_{j=0}^{q+1}(-1)^j\binom{q+1}{j}(\mathrm{ReLU}(x-j))^q.
\end{equation}

By shifting and rescaling, we can form a basis for our relevant problem. Denote the knot spacing by $h\coloneqq\frac{b-a}{m}$. We consider the feature map
\[\phi(x)=\mathbf{b}(x)\coloneqq \left[B_q\left(\frac{x-t_{1-q}}{h}\right),\ldots,B_q\left(\frac{x-t_m}{h}\right)\right]^T\]
where we extend the indexing to additional ``ghost knots'' (defined by the obvious shifts by $h$). Any degree~$q$ spline~$s$ on our domain can be expressed as a linear combination 
\begin{equation}\label{eq:bsplinebasis}
    s(x)=\sum_{i=1-q}^m \alpha_iB_q\left(\frac{x-t_i}{h}\right).
\end{equation}
Notice from the formula that the only operations we need here (in terms of the variable $x$) are sum, product, constant multiplication, and ReLU, all of which can be performed in-context exactly by the Interaction Lemma \labelcref{lem:interaction}. Consequently, the only error incurred is again from the finite-sample gap and spline regression bias.

We derive the generalization error result for degree $1$ linear splines. The following theorem assumes direct analogues of \cref{ass:polyapproximability} and \cref{ass:conditioning} for linear splines, which we state explicitly in \cref{sec:splineassumptions}. For the latter, we denote $\|\Sigma^{-1}\|\leq \tau_{\mathrm{spline}}$. Splines possess better conditioning than global high-degree polynomials which helps avoid some of the numerical instability.

\begin{thm}[Generalization Error for Spline Regression]\label{thm:splinegenerror}
    Fix the parameters of the transformer network class $\G$ as 
\begin{align*}
    & L_G = 2, m_G=O(nm), d_{\mathrm{embed}}=m+7, \ell=n+1 \\
    & L_{\mathrm{FFN}}=0, w_{\mathrm{FFN}}=0, \kappa = O(m^4R^2n^2U^2+\tau_{\mathrm{spline}}).
\end{align*}
The minimizer of empirical loss $G^\Gamma_{\hat{\theta}}$ in this class satisfies 
     \[
         \E_\Gamma[\RR(G_{\hat\theta}^\Gamma)] \leq C_3\bigg(\frac{\alpha^2\tau^2 \log m}{n}+\delta ^2\tau^2+\delta^2\bigg) +C_4\bigg(\frac{\sqrt{nm^3\log((m^4R^2n^2U^2+\tau_{\mathrm{spline}})L)}}{\sqrt{L}}\bigg)
    \]
     for absolute constants $C_3$ and $C_4$. 
\end{thm}

The proof can be found in \cref{sec:splineproof}. Besides the construction details, the proof is nearly identical to that of \cref{thm:generror}, mutatis mutandis. 

It is clear from the formula that this can be easily extended to higher degree splines with more basic arithmetic calculations. More complicated calculations can be done with a larger $d_\mathrm{embed}$, as the intermediate space can be used as a `scratchpad' for performing and storing computations in-context. We give a sketch of the construction for higher degree splines in \cref{sec:higherordersplines}.

\begin{rem}
While monomials and cardinal B-splines have simple recursive formulae, the construction can in fact be extended with a more elaborate transformer to produce general feature functions without explicit algebraic formula, introducing some error which would propagate through to the final prediction. We can adapt results of \citet{Havrilla24}, where an Interaction Lemma was used to prove a constructive universal approximation theorem of H\"older functions $f\colon [0,1]^D\to\R$ for transformers in the single task learning setting, which places the predicted function output in an entry of the final matrix. The construction follows the approach of \citet{Yarotsky} with local Taylor approximation and linear splines as a partition of unity. In our case, we need to tweak this to approximate any feature $\phi_i$, then run the construction in parallel on each $\phi_i$ and $x_i$ to featurize the data.\footnote{They perform the computations along the rows whereas we require them along the columns, but the Interaction Lemma can handle both.} Interestingly, this can still be done with constant depth independent of desired accuracy. 
\end{rem}

\subsection{Vector-Valued Functions and Multiple Variables}\label{sec:highdextensions}
Our primary analysis focused on the regression of certain real functions $\R\to \R$ for clarity. A natural question is whether we can extend this approach to a subset of continuous functions $\R^{D_1}\to \R^{D_2}$.

This framework can be easily generalized to vector-valued functions $f\colon \R\to \R^D$. By treating these as $D$ independent real functions along each output coordinate, we can construct a result by regressing along each coordinate separately. We utilize the same features (such as monomials), formed from the inputs, then approximate a vector of solutions to the least squares problems from each dimension. We give more details about the construction and resulting bounds for the polynomial case in \cref{sec:vectorvalued}.

Our framework can also be extended to multivariable functions $f\colon \R^D\to \R$, where the function class is still determined by the linear span of features. Usually the feature representation requires a more complex construction. For example, our construction for linear splines in-context in \cref{sec:splineproof} can be extended to multivariate linear splines by computing tensor products of linear splines in different directions, and still produces the basis with zero error. Similarly, the universal approximation results of \citet{Havrilla24} are established for functions on a hypercube, so they can be directly applied as well, but at the cost of error which may amplify in high-dimensions.

\section{Numerical Experiments}

\begin{figure*}[ht]
    \centering
    \begin{subfigure}[b]{0.49\linewidth}
        \centering
        \includegraphics[width=.85\linewidth]{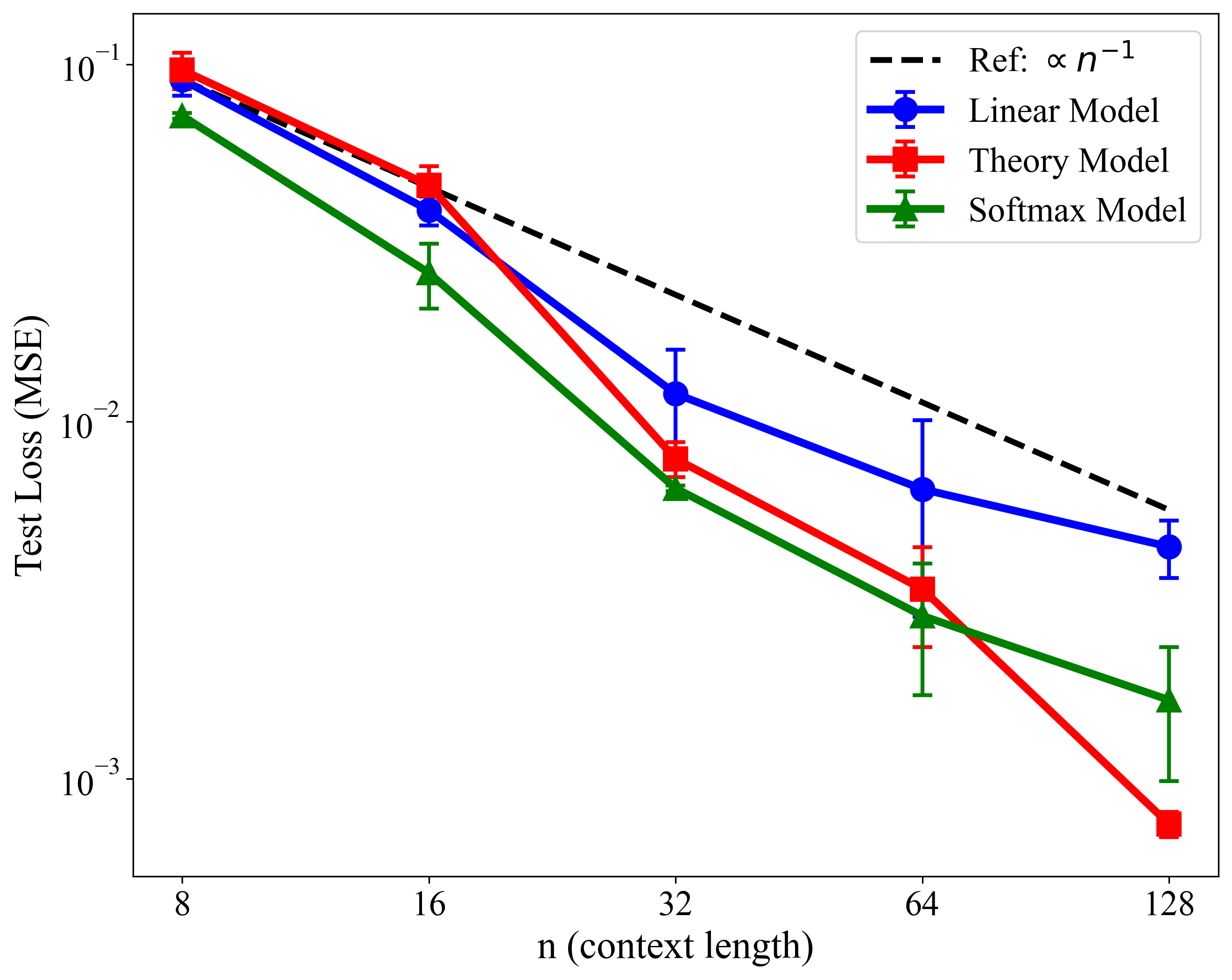}
        \caption{Test loss vs. context length ($n$) with fixed $L=32000$}
        \label{fig:nover8headsscalingn}
    \end{subfigure}
    \hfill
    \begin{subfigure}[b]{0.49\linewidth}
        \centering
        \includegraphics[width=.85\linewidth]{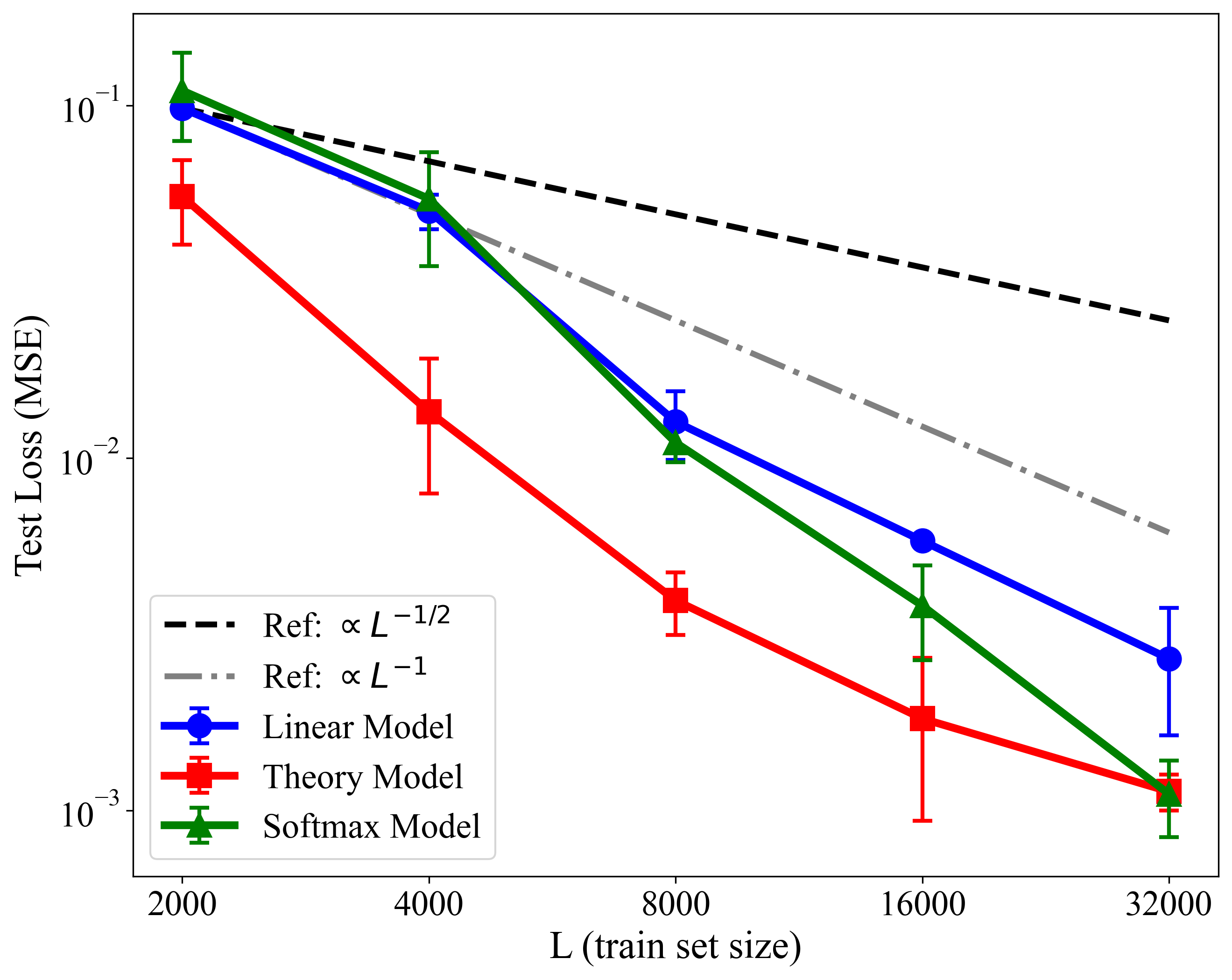}
        \caption{Test loss vs. training set size ($L$) with fixed $n=128$}
        \label{fig:nover8headsscalingL}
    \end{subfigure}
    \caption{Scaling results for different models trained and tested on synthetic regression tasks of degree $d=4$ polynomials. Models use a sum-based multi-head attention with scaling $n/8$ heads (thus (\subref{fig:nover8headsscalingL}) uses a fixed $16$ heads). Results are averaged over 3 seeds and error bars represent $\pm 1$ sd.}
    \label{fig:nover8headsscaling}
\end{figure*}

We conduct numerical experiments on synthetic datasets to verify our predicted scaling laws, primarily focused on polynomial regression as in \cref{thm:generror}. We restrict our function space to polynomials of some fixed degree $d$ and generate random polynomials by uniformly sampling coefficients on $[-1,1]$ under a Legendre polynomial basis (orthogonality ensures uniformity over the function space). To generate prompts, we uniformly sample $n$ context points from $[-1,1]$ and compute the corresponding outputs using the sampled polynomial. We also uniformly sample a query from $[-1,1]$. 

For the architecture, we remain faithful to what we used for the theory in \cref{def:transformernetwork}, with fixed embedding and positional encoding that hardcodes the prompt into \eqref{eq:encoding}, and decoding that outputs the relevant entry. We use a custom sum-based multi-head attention as defined in \eqref{eq:MHA}, and test three architectures with different activations: a theory model consistent with the construction for Theorem \ref{thm:generror}, which uses ReLU blocks and a single linear layer at the end (one head, no FFN); a linear transformer with only linear blocks; and a softmax transformer with only softmax blocks. For simplicity, we maintain an equal number of heads in each block, except the last linear layer in the theory model. The hyperparameters are informed by the theoretical construction used in our approximation error \cref{lem:approxerror}, with the specific choices listed in \cref{tab:hyperparams} in the appendix. 

Note that our approximation theory requires the number of attention heads per block to scale with the context length, where more heads are used to process additional tokens. This differs from practice, where a single transformer is trained and infers on prompts of various lengths. Our construction values accuracy over efficiency and almost certainly calls for more heads than are necessary in practice. We verify this in the numerical experiments, testing both a growing (with $n$) number of heads and a fixed number, where we observe similar scaling results in both.

Our results with a scaling number of heads are shown in \cref{fig:nover8headsscaling}. The result with a fixed $4$ heads can be seen at \cref{fig:4headsscalingn} in \cref{sec:additionalnumerical}, along with more ablation studies. We observe scaling that confirms our theoretical upper bounds; the scaling with training set size exhibits a faster convergence rate of $1/L$ than our theoretical bound of $1/\sqrt{L}$, and we leave the study of optimality to future work. The absolute error throughout is comparable for the theory model (ReLU + linear) and the softmax model, suggesting in this case the choice of activation is inconsequential. The linear model starts with similar error for small $n$ or $L$, but seems to scale worse than the other two. 

Finally, we also verify our scaling law for spline regression in \cref{thm:splinegenerror}. The results can be found in \cref{fig:splinescaling} in Appendix \ref{sec:additionalnumerical}.

\section{Conclusion}
In this work, we establish rigorous generalization error bounds for nonlinear in-context regression, identifying explicit scaling laws with respect to context length and training data size. A central aspect of our analysis is a featurization-based perspective: we decompose the error into statistical and approximation components and explicitly construct transformers that perform regression through fixed feature representations, such as polynomial or spline bases. From this viewpoint, attention mechanisms serve a dual role, simultaneously realizing the feature representation and approximating the solution of the resulting linear system within a single end-to-end architecture. Numerical experiments on synthetic regression tasks corroborate our theoretical predictions and illustrate the effectiveness of this framework.

\section*{Acknowledgment}
Alexander Hsu and Rongjie Lai are supported in part by NSF DMS-2401297. Zhaiming Shen and Wenjing Liao are partially supported by National Science Foundation under
NSF DMS-2145167 and the U.S. Department of Energy under DOE-SC0024348.

\bibliographystyle{plainnat}
\bibliography{references}

\newpage
\appendix
\section{Relevant Lemmas}

Below are two important lemmas used in the construction of the feature representation. We refer the reader to \citep{Shen2} for the proofs.

\begin{lemma} [Interaction Lemma] \label{lem:interaction}
   Let $H=[\mathbf{h}_t]_{1\leq t\leq\ell}\in\mathbb{R}^{d_{\mathrm{embed}}\times\ell}$ be a matrix such that \newline $\mathbf{h}_t^{(d_{\mathrm{embed}}-2)\colon (d_{\mathrm{embed}}-1)}=\mathcal{I}_t$ and $h^{d_{\mathrm{embed}}}_t=1$ for all $1\leq t\leq \ell$. Fix $t_1,t_2$ such that $1\leq t_1,t_2\leq \ell$, and $i$ satisfying $1\leq i\leq d_{\mathrm{embed}}$. Suppose $d_{\mathrm{embed}}\geq 5$ and $\|H\|_{\max}<U$ for some $U>0$, and the \textbf{data kernels} $Q^{\mathrm{data}}\in \R^{(d_\mathrm{embed}-3)\times d_\mathrm{embed}}$ (all but the last three rows in $Q$) and $K^{\mathrm{data}}\in \R^{(d_\mathrm{embed}-3)\times d_\mathrm{embed}}$ (all but the last three rows in $K$) satisfy $\max\{\|Q^{\mathrm{data}}\|_{\max},\|K^{\mathrm{data}}\|_{\max}\}\leq\mu$. Then one can construct a ReLU-activated attention head $A$ with weights satisfying $\|\theta_A\|_{\infty}=O(d_\mathrm{embed}^4\mu^2\ell^2U^2)$ such that 
   \[ [A(H)]_t = 
   \begin{cases} 
    \sigma(\langle Q^{\mathrm{data}}\mathbf{h}_{t_1}, K^{\mathrm{data}}\mathbf{h}_{t_2}\rangle)\mathbf{e}_i & \text{if}\text{  } t=t_1, \\
      0 & \text{otherwise}. \\
   \end{cases}
\]
\end{lemma}

In other words, specifying $Q^\mathrm{data}$ and $K^\mathrm{data}$ properly allows us to select and \emph{interact} any pair of entries in the matrix (via inner product), adding the result to an entry of our choosing. Furthermore, we may scale this output by any real number by scaling the parameter $V$ correspondingly. To retain negative outputs through ReLU, we may incorporate a large enough constant shift in the interaction, then subtract off this constant in the subsequent feedforward component via the following lemma. 

\begin{lemma}[Decrementing Lemma] \label{lem:decrementing}
Let $d_{\mathrm{embed}}\geq 5$, and let $H=[\mathbf{h}_t]_{1\leq t\leq\ell}\in\mathbb{R}^{d_{\mathrm{embed}}\times\ell}$ be a matrix such that $\mathbf{h}_t^{(d_{\mathrm{embed}}-2)\colon (d_{\mathrm{embed}}-1)}=\mathcal{I}_t$ and $h^{d_{\mathrm{embed}}}_t=1$ for all $1\leq t\leq \ell$. Then for any $r_1,r_2$ with  $1\leq r_1\leq r_2\leq d_{\mathrm{embed}}-3$, any $k_1,k_2$ with $1\leq k_1\leq k_2\leq \ell$, and any $M>0$, there exists a 6-layer residual feedforward network $({\mathrm{FFN}})$ such that
\[ \mathrm{FFN}(\mathbf{h}_t) + \mathbf{h}_t = 
   \begin{cases} 
    \mathbf{h}_t & \text{if}\text{  } t\in\{1,\cdots,k_1\}\cup\{k_2,\cdots,\ell\} \\
      \begin{bmatrix}
      (\mathbf{h}_t)_{1} \\
          \vdots \\
          (\mathbf{h}_t)_{r_1-1} \\
          (\mathbf{h}_t)_{r_1}-M \\
          \vdots \\
          (\mathbf{h}_t)_{r_2}-M \\
          (\mathbf{h}_t)_{r_2+1} \\
          \vdots \\
          (\mathbf{h}_t)_{d_{\mathrm{embed}}-3} \\
          \mathcal{I}_t \\
          1
      \end{bmatrix} & \text{otherwise}\text{ }  \\
   \end{cases}
\]
Additionally, the weights satisfy $\|\theta_{\mathrm{FFN}}\|_{\infty}\leq O(\ell M)$.
\end{lemma}

We also include the statement of the matrix Bernstein inequality \citep{Tropp}, useful in analyzing the gap between empirical and population covariance.

\begin{prop}[Matrix Bernstein]\label{thm:bernstein}

Let $S_1, \ldots, S_n$ be independent, centered random matrices with common dimension $d_1 \times d_2$, and assume that they are uniformly bounded:

$$
\mathbb{E} [S_k]=0 \text { and }\left\|S_k\right\| \leq L \text { for each } k=1, \ldots, n .
$$

Introduce the random matrix

$$
Y=\sum_{k=1}^n S_k
$$

and let $v(Y)$ denote the matrix variance statistic of the sum:

$$
\begin{aligned}
v(Y) & =\max \left\{\left\|\mathbb{E}\left[Y Y^*\right]\right\|,\left\|\mathbb{E}\left[Y^* Y\right]\right\|\right\} \\
& =\max \left\{\left\|\sum_{k=1}^n \mathbb{E}\left[S_k S_k^*\right]\right\|,\left\|\sum_{k=1}^n \mathbb{E}\left[S_k^* S_k\right]\right\|\right\} .
\end{aligned}
$$

Then

$$
\mathbb{E}\|Y\| \leq \sqrt{2 v(Y) \log \left(d_1+d_2\right)}+\frac{1}{3} L \log \left(d_1+d_2\right).
$$

Furthermore, for all $t\geq 0$, 

$$
\mathbb{P}\{\|Y\| \geq t\} \leq\left(d_1+d_2\right) \cdot \exp \left(\frac{-t^2 / 2}{v(Y)+L t / 3}\right).
$$

\end{prop}

\section{Supplementary Materials for \cref{sec:generror}}

\subsection{Summary of Constants and Variables in \cref{thm:generror}}
\begin{table}[h]
    \centering
    \caption{Summary of Notation and Bounds}
    \label{tab:notation}
    \begin{tabular}{cl}
        \toprule
        \textbf{Symbol} & \textbf{Description} \\
        \midrule
        $R$ & Upper bound on inputs; $X\subseteq[-R, R]$ \\
        $d$ & Degree of polynomial regression \\
        $R_{\mathbf{v}}$ & Upper bound on $\mathbf{v}(x)$ \\
        $U$ & Combined bound $U=\max(R_{\mathbf{v}}, R_F)$, where $R_F$ is upper bound on $f(x)$ \\
        $\tau$ & Upper bound on the uncentered covariance inverse $\|\Sigma^{-1}\|$ \\
        $\delta$ & $L^\infty$-norm approximation error of $f$ by a degree $\leq d$ polynomial \\
        $\alpha$ & Upper bound on the polynomial coefficients \\
        $n$ & Context length \\
        $L$ & Size of training set $\Gamma$ \\
        \bottomrule
    \end{tabular}
\end{table}

\subsection{Proof of \cref{lem:approxerror}}\label{sec:approxerrorproof}
By \cref{ass:polyapproximability}, for any $f\in F(\delta,\alpha,d,R_F)$ there exists a degree $\leq d$ polynomial $p_{f}(x)$ with coefficients satisfying $\|\mathbf{a}\|\leq \alpha$ such that 
\[\|f-p_{f}\|_{L^\infty(X)}\leq \delta.\]

The approximation error decomposes as
\begin{align*}
    \RR(G_{\tilde{\theta}}) &= \E_{\mathbf{s}^f}|G_{\tilde{\theta}}(\mathbf{s}^f)-f(x)|^2 \\
    &= \E_{\mathbf{s}^f}|G_{\tilde{\theta}}(\mathbf{s}^f)-p_{f}(x)+p_{f}(x)-f(x)|^2 \\
    &\leq 2\E_{\mathbf{s}^f}|G_{\tilde{\theta}}(\mathbf{s}^f)-p_{f}(x)|^2  +2\E_{x\sim \D_X, f\sim \D_{F}} |p_{f}(x)-f(x)|^2 \\
    &\leq 2\E_{\mathbf{s}^f}|G_{\tilde{\theta}}(\mathbf{s}^f)-p_{f}(x)|^2 + 2\delta^2.
\end{align*}  

It remains to bound the first term. We first explicitly construct a transformer network $G_{\tilde{\theta}}$ to approximate $p_f$, and show how it transforms an arbitrary input $\mathbf{s}^f=(x_1,y_1,\ldots, x_n,y_n,x)$ where $y_i=f(x_i)$. For notational simplicity, we may sometimes extend the indexing to the query $x$ as $x_{n+1}$. 

We begin by embedding our input as 
\[\mathrm{PE+E}(\mathbf{s}^f)=\begin{bmatrix}
    x_1 & x_2 & x_3 & \cdots & x_n & x \\
    0 & 0 & 0 & \cdots & 0 & 0 \\
    \vdots &  &  & \ddots & & \vdots \\
    0 & 0 & 0 & \cdots & 0 & 0 \\
    y_1 & y_2 & y_3 & \cdots & y_n & 0 \\
    0 & 0 & 0 & \cdots & 0 & 0 \\
    \mathcal{I}_1 & \mathcal{I}_2 & \mathcal{I}_3 & \cdots & \mathcal{I}_n & \mathcal{I}_{n+1} \\
    1 & 1 & 1 & \cdots & 1 & 1
\end{bmatrix} \in \R^{d_\mathrm{embed}\times (n+1)}.\]

We will refer to columns (tokens) of similar such matrices as $\mathbf{h}_i$. This particular construction necessitates $d_\mathrm{embed}=d+7$, but we retain the general notation $d_\mathrm{embed}$ throughout the proof, which makes it clearer how the result changes for other settings. 

Our next goal is to repeatedly apply the Interaction Lemma and the Decrementing Lemma (\cref{lem:interaction,lem:decrementing}) to construct a composition of transformer blocks with ReLU attention that ultimately outputs the matrix

\[
\begin{bmatrix}
    x_1 & x_2 & x_3 & \cdots & x_n & x \\
    1 & 1 & 1 & \cdots & 1 & 1 \\
    x_1 & x_2 & x_3 & \cdots & x_n & x \\
    x_1^2 & x_2^2 & x_3^2 & \cdots & x_n^2 & x^2\\
    \vdots &  &  & \ddots & & \vdots \\
    x_1^d & x_2^d & x_3^d & \cdots & x_n^d & x^d \\
    y_1 & y_2 & y_3 & \cdots & y_n & 0 \\
    0 & 0 & 0 & \cdots & 0 & 0 \\
    \mathcal{I}_1 & \mathcal{I}_2 & \mathcal{I}_3 & \cdots & \mathcal{I}_n & \mathcal{I}_{n+1} \\
    1 & 1 & 1 & \cdots & 1 & 1
\end{bmatrix}.\]

The first step is to complete the degree $0$ and $1$ rows by copying entries from the bottom bias row and the top input row respectively, where each head uses an Interaction Lemma construction for constant multiplication by $1$. We denote by $\mathbf{e}_i$ the $i$-th standard basis vector $(0,\ldots, 0, 1, 0, \ldots, 0)^T$. To copy the bias row, we define attention heads $A_i$ for $1\leq i\leq n+1$ by choosing $V_i=\mathbf{e}_2\mathbf{e}_{d_\mathrm{embed}}^T$ and data kernels

\begin{equation*}
Q_i^{\mathrm{data}} = 
\left[\begin{array}{ccccc|ccc} 
    0 & & & & & 0 & 0 & 1 \\ 
    & 0 & & & & 0 & 0 & 0 \\
    & & \ddots & & & \vdots & \vdots & \vdots \\
    & & & 0 & & 0 & 0 & 0 \\
    & & & & 0 & 0 & 0 & 0
\end{array}\right] 
\quad\quad\quad 
K_i^{\mathrm{data}} = 
\left[\begin{array}{ccccc|ccc} 
    0 & & & & & 0 & 0 & 1 \\ 
    & 0 & & & & 0 & 0 & 0 \\
    & & \ddots & & & \vdots & \vdots & \vdots \\
    & & & 0 & & 0 & 0 & 0 \\
    & & & & 0 & 0 & 0 & 0
\end{array}\right]  
\end{equation*}

in $Q^{\mathrm{data}}_i,K^{\mathrm{data}}_i\in\mathbb{R}^{(d_\mathrm{embed}-3)\times d_\mathrm{embed}}$. By \cref{lem:interaction} there exist attention heads $A_i$ for $1\leq i\leq n+1$ satisfying 
\[A_i(\mathbf{h}_i)=\sigma(\langle Q^{\mathrm{data}}_i\mathbf{h}_i, K^{\mathrm{data}}_i\mathbf{h}_i\rangle) V_i\mathbf{h}_i = \sigma(1\cdot 1)\mathbf{e}_2 = \mathbf{e}_2\]
and $A_i(\mathbf{h}_j)=0$ when $j\neq i$. A similar method works for copying the $x_i$ row: for $1\leq i'\leq n+1$, let $V_{i'}=\mathbf{e}_3\mathbf{e}_{d_\mathrm{embed}}^T$ and select data kernels

\begin{equation*}
Q_{i'}^{\mathrm{data}} = 
\left[\begin{array}{ccccc|ccc} 
    0 & & & & & 0 & 0 & 1 \\ 
    & 0 & & & & 0 & 0 & 0 \\
    & & \ddots & & & \vdots & \vdots & \vdots \\
    & & & 0 & & 0 & 0 & 0 \\
    & & & & 0 & 0 & 0 & 1
\end{array}\right] 
\quad\quad\quad 
K_{i'}^{\mathrm{data}} = 
\left[\begin{array}{ccccc|ccc} 
    1 & & & & & 0 & 0 & 0 \\ 
    & 0 & & & & 0 & 0 & 0 \\
    & & \ddots & & & \vdots & \vdots & \vdots \\
    & & & 0 & & 0 & 0 & 0 \\
    & & & & 0 & 0 & 0 & M
\end{array}\right]  
\end{equation*}

where $Q_{i'}^{\mathrm{data}}$ has a 1 in index $(1,d_\mathrm{embed})$ and the bottom right corner, and $K_{i'}^{\mathrm{data}}$ has a $1$ in the top left corner and an $M$ in the bottom right corner.  \cref{lem:interaction} yields attention heads $A_{i'}$ for $1\leq i'\leq n+1$ such that 
\[A_{i'}(\mathbf{h}_{i'})=\sigma(\langle Q^{\mathrm{data}}_{i'}\mathbf{h}_{i'}, K^{\mathrm{data}}_{i'}\mathbf{h}_{i'}\rangle) V_{i'}\mathbf{h}_{i'} = \sigma(x_{i'}+M)\mathbf{e}_3 = (x_{i'}+M)\mathbf{e}_3\]
and $A_{i'}(\mathbf{h}_j)=0$ when $j\neq i'$. Note $M$ is chosen large enough to preserve negative values through the ReLU, so in this case we require $x_i+M\geq 0$ (we may choose some uniform $M$ sufficient for all future steps, which does not affect the final weight bound). The residual MHA summed over all heads $A_i$ and $A_{i'}$ returns the second and third rows completed, with an additional $+M$ on every entry of the third row. By \cref{lem:decrementing}, we can subtract this off with the subsequent FFN component, which completes the first transformer block $\mathrm{B}_0\in\mathcal{B}(n+1,6,d_\mathrm{embed})$:

\[\mathrm{B}_0(\mathrm{PE+E}(\mathbf{s}^f))=H\coloneqq \begin{bmatrix}
    x_1 & x_2 & x_3 & \cdots & x_n & x \\
    1 & 1 & 1 & \cdots & 1 & 1 \\
    x_1 & x_2 & x_3 & \cdots & x_n & x \\
    0 & 0 & 0 & \cdots & 0 & 0 \\
    \vdots &  &  & \ddots & & \vdots \\
    0 & 0 & 0 & \cdots & 0 & 0 \\
    y_1 & y_2 & y_3 & \cdots & y_n & 0 \\
    0 & 0 & 0 & \cdots & 0 & 0 \\
    \mathcal{I}_1 & \mathcal{I}_2 & \mathcal{I}_3 & \cdots & \mathcal{I}_n & \mathcal{I}_{n+1} \\
    1 & 1 & 1 & \cdots & 1 & 1
\end{bmatrix}.\]

We then construct an attention head which allows us to multiply two elements and add the result to a chosen entry of the matrix, which we extend to complete the $x_i^2$ row from the $x_i$ row. For full generality, we assume $x_i^2$ could be negative, as we want the construction to work for the multiplication of \emph{any} two elements. 

Again, let us define each attention head $A_i$ for $1\leq i\leq n+1$. Let $V_i=\mathbf{e}_4\mathbf{e}_{d_{\mathrm{embed}}}^T$ and choose data kernels of the form 

\begin{equation*}
Q^{\mathrm{data}}_i = 
\left[\begin{array}{cccccccc|ccc} 
    0 & & & & & & & &  0 & 0 & 0 \\ 
    & 0 & & & & & & & \vdots & \vdots & \vdots \\
    & & 1 & & & & & & 0 & 0 & 0 \\ 
    & & & 0 & & & & & 0 & 0 & 0 \\
	& & & & \ddots & & & & 0 & 0 & 0 \\ 
    & & & & & 0 & & & \vdots & \vdots & \vdots\\
    & & & & & & 0 & & 0 & 0 & 0 \\
    & & & & & & & 0 & 0 & 0 & 1
\end{array}\right] 
\quad
K^{\mathrm{data}}_i = 
\left[\begin{array}{cccccccc|ccc} 
    0 & & & & & & & &  0 & 0 & 0 \\ 
    & 0 & & & & & & & \vdots & \vdots & \vdots \\
    & & 1 & & & & & & 0 & 0 & 0 \\ 
    & & & 0 & & & & & 0 & 0 & 0 \\
	& & & & \ddots & & & & 0 & 0 & 0 \\ 
    & & & & & 0 & & & \vdots & \vdots & \vdots\\
    & & & & & & 0 & & 0 & 0 & 0 \\
    & & & & & & & 0 & 0 & 0 & M
\end{array}\right] 
\end{equation*}
where $Q^{\mathrm{data}}_i,K^{\mathrm{data}}_i\in\mathbb{R}^{(d_\mathrm{embed}-3)\times d_\mathrm{embed}}$. Both $Q^{\mathrm{data}}_i$ and $K^{\mathrm{data}}_i$ have a 1 in the $(3,3)$ entry; $Q^{\mathrm{data}}_i$ has the bottom right entry equal to 1 and $K^{\mathrm{data}}_i$ has a nonzero positive constant $M$ in the bottom right corner. All other entries are zero. By \cref{lem:interaction}, we can construct attention heads $A_i$ for all $1\leq i\leq n+1$ such that $\mathbf{h}_i$ only interacts with itself, i.e., 

\begin{equation*}
    A_i(\mathbf{h}_i) = \sigma(\langle Q^{\mathrm{data}}_i\mathbf{h}_i, K^{\mathrm{data}}_i\mathbf{h}_i\rangle) V_i\mathbf{h}_i = \sigma(x^2_{i}+M)\mathbf{e}_4 = (x_i^2+M)\mathbf{e}_4
\end{equation*}
and $A_i(\mathbf{h}_j)=0$ when $j\neq i$. Thus the $i$-th attention head $A_i$ returns

\begin{equation*}
    {\mathrm A_i}(H)+H = \begin{bmatrix}
    x_1 & \cdots & x_{i-1} & x_i & x_{i+1} & \cdots & x_n & x \\
    1 & \cdots & 1 & 1 & 1 & \cdots & 1 & 1 \\
    x_1 & \cdots & x_{i-1} & x_i & x_{i+1} & \cdots & x_n & x \\
    0 & \cdots & 0 & x_i^2+M & 0 & \cdots & 0 & 0 \\
    0 & \cdots & 0 & 0 & 0 & \cdots & 0 & 0 \\
    \vdots &  &  & \ddots & & & & \vdots \\
    0 & \cdots & 0 & 0 & 0 & \cdots & 0 & 0 \\
    y_1 & \cdots & y_{i-1} & y_i & y_{i+1} & \cdots & y_n & 0 \\
    0 & \cdots & 0 & 0 & 0 & \cdots & 0 & 0 \\
    \mathcal{I}_1 & \cdots & \mathcal{I}_{i-1} & \mathcal{I}_i & \mathcal{I}_{i+1} & \cdots & \mathcal{I}_n & \mathcal{I}_{n+1} \\
    1 & \cdots & 1 & 1 & 1 & \cdots & 1 & 1
\end{bmatrix}.
\end{equation*}

Notice the nonzero indices in the left part of $Q^\mathrm{data}$ and $K^\mathrm{data}$ are used to select the relevant rows of $H$, the columns are selected with the Interaction Lemma, and the choice of $V$ determines which row the final result is put in. The residual multi-head attention yields 

\begin{equation*}
    \mathrm{MHA}(H)+H = \begin{bmatrix}
    x_1 & x_2 & x_3 & \cdots & x_n & x \\
    1 & 1 & 1 & \cdots & 1 & 1 \\
    x_1 & x_2 & x_3 & \cdots & x_n & x \\
    x_1^2+M & x_2^2+M & x_3^2+M & \cdots & x_n^2+M & x^2+M \\
    0 & 0 & 0 & \cdots & 0 & 0 \\
    \vdots &  &  & \ddots & & \vdots \\
    0 & 0 & 0 & \cdots & 0 & 0 \\
    y_1 & y_2 & y_3 & \cdots & y_n & 0 \\
    0 & 0 & 0 & \cdots & 0 & 0 \\
    \mathcal{I}_1 & \mathcal{I}_2 & \mathcal{I}_3 & \cdots & \mathcal{I}_n & \mathcal{I}_{n+1} \\
    1 & 1 & 1 & \cdots & 1 & 1
\end{bmatrix}.
\end{equation*}

By \cref{lem:decrementing}, we can subtract off the additional constant with a FFN. Therefore, there exists a transformer block ${\mathrm B}_1\in\mathcal{B}(n+1,6,d_\mathrm{embed})$ which produces

\begin{equation*}
    {\mathrm B}_1(H) = \begin{bmatrix}
    x_1 & x_2 & x_3 & \cdots & x_n & x \\
    1 & 1 & 1 & \cdots & 1 & 1 \\
    x_1 & x_2 & x_3 & \cdots & x_n & x \\
    x_1^2 & x_2^2 & x_3^2 & \cdots & x_n^2 & x^2 \\
    0 & 0 & 0 & \cdots & 0 & 0 \\
    \vdots &  &  & \ddots & & \vdots \\
    0 & 0 & 0 & \cdots & 0 & 0 \\
    y_1 & y_2 & y_3 & \cdots & y_n & 0 \\
    0 & 0 & 0 & \cdots & 0 & 0 \\
    \mathcal{I}_1 & \mathcal{I}_2 & \mathcal{I}_3 & \cdots & \mathcal{I}_n & \mathcal{I}_{n+1} \\
    1 & 1 & 1 & \cdots & 1 & 1
\end{bmatrix}.
\end{equation*}

The most straightforward way to proceed here sequentially fills the matrix by chaining similar transformer blocks, which would use a total of $d-1$ blocks. We give a shallower construction, using $O(\log d)$ total blocks. We still require the same number of total attention heads, but we do more computations in parallel, with the later layers being wider. It suffices to work with the case where $d=2^D$ for some $D\in\N$. Our second layer will construct both the third power and the fourth power row by multiplying $x_i^2$ with $x_i$ to obtain $x_i^3$, and with itself to obtain $x_i^4$. We repeat this: in the $j$-th layer, we multiply the $x_i^{2^j}$ term by the different $x_i^k$ for $1 \leq k \leq 2^j$, which completes up to the $2^{j+1}$-th power. Thus after the $j$-th layer we are left with the matrix

\begin{equation*}
    {\mathrm B}_{j}\circ\cdots\circ {\mathrm B}_1(H) = \begin{bmatrix}
    x_1 & x_2 & x_3 & \cdots & x_n & x \\
    1 & 1 & 1 & \cdots & 1 & 1 \\
    x_1 & x_2 & x_3 & \cdots & x_n & x \\
    x_1^2 & x_2^2 & x_3^2 & \cdots & x_n^2 & x^2 \\
    \vdots &  &  & \ddots & & \vdots \\
    x_1^{2^{j+1}} & x_2^{2^{j+1}} & x_3^{2^{j+1}} & \cdots & x_n^{2^{j+1}} & x^{2^{j+1}} \\
    0 & 0 & 0 & \cdots & 0 & 0 \\
    \vdots &  &  & \ddots & & \vdots \\
    0 & 0 & 0 & \cdots & 0 & 0 \\
    y_1 & y_2 & y_3 & \cdots & y_n & 0 \\
    0 & 0 & 0 & \cdots & 0 & 0 \\
    \mathcal{I}_1 & \mathcal{I}_2 & \mathcal{I}_3 & \cdots & \mathcal{I}_n & \mathcal{I}_{n+1} \\
    1 & 1 & 1 & \cdots & 1 & 1
\end{bmatrix},
\end{equation*}

Note that the $j$-th layer consists of a total of $2^j\cdot (n+1)$ attention heads. With $d=2^D$ we require $D-1$ total layers, thus for arbitrary $d$ this construction takes $O(\log d)$ layers. Furthermore, the maximum number of attention heads in a block is $O(dn)$ heads (achieved in the $(\log d-1)$-th block). Altogether, we can construct ${\mathrm B}_1,{\mathrm B}_2, \cdots, {\mathrm B}_{O(\log d)}$ such that 

\begin{equation*}
    Z\coloneqq {\mathrm B}_{O(\log d)}\circ\cdots\circ {\mathrm B}_1(H) = \begin{bmatrix}
    x_1 & x_2 & x_3 & \cdots & x_n & x \\
    1 & 1 & 1 & \cdots & 1 & 1 \\
    x_1 & x_2 & x_3 & \cdots & x_n & x \\
    x_1^2 & x_2^2 & x_3^2 & \cdots & x_n^2 & x^2 \\
    \vdots &  &  & \ddots & & \vdots \\
    x_1^d & x_2^d & x_3^d & \cdots & x_n^d & x^d \\
    y_1 & y_2 & y_3 & \cdots & y_n & 0 \\
    0 & 0 & 0 & \cdots & 0 & 0 \\
    \mathcal{I}_1 & \mathcal{I}_2 & \mathcal{I}_3 & \cdots & \mathcal{I}_n & \mathcal{I}_{n+1} \\
    1 & 1 & 1 & \cdots & 1 & 1
\end{bmatrix}.
\end{equation*}

According to \cref{lem:interaction,lem:decrementing}, the parameters used thus far are bounded by $\|\theta_{\mathrm{B}_{0,\ldots O(\log d)}}\|_\infty=O(d_\mathrm{embed}^4R_\mathbf{v}^2n^2U^2)$. Again, note that we incur no error in this part. This matrix can be rewritten succinctly as 

\[
    Z= \begin{bmatrix}
    x_1 & \cdots & x_n & x \\
    \mathbf{v}(x_1) & \cdots & \mathbf{v}(x_n) & \mathbf{v}(x) \\
    y_1 & \cdots & y_n & 0 \\
    0 & \cdots & 0 & 0 \\
    \mathcal{I}_1 & \cdots & \mathcal{I}_n & \mathcal{I}_{n+1} \\
    1 & \cdots & 1 & 1
    \end{bmatrix}
\]

where $\mathbf{v}(x_i)\coloneqq [1,x_i,x_i^2,\ldots,x_i^d]^T$ are the feature vectors. We then proceed to the final transformer block, which approximates the solution to the least squares problem by a linear attention layer. The rows below the $y_i$ are not used in this next step.

Under the parameterization
\begin{equation}\label{eq:monomiallinearparams}
    V=\begin{bmatrix}
    0_{(d+2)\times (d+2)} & 0_{(d+2)\times 1} & 0_{(d+2)\times 4} \\
    0_{1\times (d+2)} & 1 & 0_{1\times 4} \\
    0_{4\times (d+2)} & 0_{4\times 1} & 0_{4\times 4}
    \end{bmatrix}\qquad \text{and} \qquad Q=\begin{bmatrix}
    0_{1\times 1} & 0_{1\times 1} & 0_{1\times 1} \\
    0_{(d+1)\times 1} & \tilde{Q} & 0_{(d+1)\times 5} \\
    0_{5\times 1} & 0_{5\times (d+1)} & 0_{5\times 5}
\end{bmatrix}
\end{equation}
in $\R^{d_\mathrm{embed}\times d_\mathrm{embed}}$ where $\tilde{Q}\in\R^{(d+1)\times (d+1)}$, a simple derivation shows that the $(d+3,n+1)$-th entry, which is read off by the subsequent decoder, becomes
\begin{equation}\label{eq:finaloutput} \left(\frac{1}{n}\sum_{i=1}^ny_i\mathbf{v}(x_i)^T\right)\tilde{Q}\mathbf{v}(x).\end{equation}
This is the final output of the entire transformer.

Recall $\mathbf{a}\in \R^{d+1}$, the coefficients of the approximating polynomial guaranteed to exist by \cref{ass:polyapproximability}. The close approximation condition is equivalent to writing $f(x)=\mathbf{a}^T\mathbf{v}(x)+m(x)$ on $X$, where $\|m\|_{L^\infty(X)}\leq \delta$. In particular, for each $1\leq i\leq n$ we can write $y_i=\mathbf{a}^T \mathbf{v}(x_i)+m_i$ where $m_i=m(x_i)$ and $|m_i|\leq \delta$. 

Substituting this assumption into \eqref{eq:finaloutput}, we get
\[\left(\frac{1}{n}\sum_{i=1}^n(\mathbf{a}^T \mathbf{v}(x_i)+m_i)\mathbf{v}(x_i)^T\right)\tilde{Q}\mathbf{v}(x)=\mathbf{a}^T\Sigma_n\tilde Q\mathbf{v}(x) + \frac{1}{n}\left(\sum_{i=1}^nm_i\mathbf{v}(x_i)^T\right)\tilde{Q}\mathbf{v}(x)\]
where $\Sigma_n\coloneqq \frac{1}{n}\sum_{i=1}^n\mathbf{v}(x_i)\mathbf{v}(x_i)^T$ denotes the \emph{empirical uncentered covariance}.
 
Expressing our loss in terms of this output gives
\begin{align}\label{eq:approxerrordecomp}
\nonumber \E_{\mathbf{s}^f}|G_{\tilde{\theta}}(\mathbf{s}^f)-p_{f}(x)|^2
&=\E_{\mathbf{s}^f}|\mathbf{a}^T\Sigma_n\tilde Q\mathbf{v}(x)+M-\mathbf{a}^T\mathbf{v}(x)|^2\\
&\leq2\E_{\mathbf{s}^f}|\mathbf{a}^T(\Sigma_n\tilde{Q}-I)\mathbf{v}(x)|^2+2\E_{\mathbf{s}^f}\bigg|\frac{1}{n}\left(\sum_{i=1}^nm_i\mathbf{v}(x_i)^T\right)\tilde{Q}\mathbf{v}(x)\bigg|^2.
\end{align}

The population form of the uncentered covariance is defined as $\Sigma\coloneqq \E[\mathbf{v}(x)\mathbf{v}(x)^T]$. We choose $\tilde{Q}=\Sigma^{-1}$, inspired by the line above since $\Sigma_n\approx \Sigma$ for large $n$. In other words, this choice of $\tilde{Q}$ makes the transformer output an approximation to the result of polynomial regression. The weights for the whole transformer can be bounded by $O(d_\mathrm{embed}^4R_\mathbf{v}^2n^2U^2+\tau)$.

Consider the second term in \eqref{eq:approxerrordecomp}, the additional error incurred from polynomial fitting that propagates through to the least squares solution, which can be easily managed by the uniform bounds on each component. By submultiplicativity and the triangle inequality, substituting in each corresponding bound, we get the following:
\begin{align*}
\E_{\mathbf{s}^f}\bigg|\frac{1}{n}\left(\sum_{i=1}^nm_i\mathbf{v}(x_i)^T\right)\Sigma^{-1}\mathbf{v}(x)\bigg|^2 &\leq \E_{\mathbf{s}^f}\left(\frac{1}{n}\|\sum_{i=1}^nm_i\mathbf{v}(x_i)^T\|\|\Sigma^{-1}\|\|\mathbf{v}(x)\|\right)^2
\\ &\leq \E_{\mathbf{s}^f}\left(\frac{1}{n}\left(\sum_{i=1}^n|m_i|\|\mathbf{v}(x_i)^T\|\right)\|\Sigma^{-1}\|\|\mathbf{v}(x)\|\right)^2
\\ & \leq (\delta R_\mathbf{v}^2\tau)^2.
\end{align*}

The first term in \eqref{eq:approxerrordecomp} becomes 
\[\E_{\mathbf{s}^f}|\mathbf{a}^T(\Sigma_n\Sigma^{-1}-I)\mathbf{v}(x)|^2.\] 
We have uniform bounds on $\mathbf{a}$ and $\mathbf{v}$. We control the size of the middle multiplicand using the matrix Bernstein inequality \cref{thm:bernstein}, which is its canonical application \citep{Tropp}. Submultiplicativity yields 
\[\|\Sigma_n\Sigma^{-1}-I\|\leq \|\Sigma^{-1}\|\|\Sigma_n-\Sigma\|.\]
Again, recall we have $\|\Sigma^{-1}\|\leq \tau$. By Jensen's inequality, we can derive a uniform spectral bound on $\Sigma$ as 
\[\|\Sigma\|=\|\E[\mathbf{v}(x)\mathbf{v}(x)^T]\|\leq \E\|\mathbf{v}(x)\mathbf{v}(x)^T\|\leq \E\|\mathbf{v}(x)\|^2\leq R_{\mathbf{v}}^2.\] 

We define the random variable 
\[Y\coloneqq \Sigma_n-\Sigma\]
as the difference between the empirical and the true uncentered covariance. 
Note that by definition of $\Sigma_n$ we can write $Y$ as the sum 
\[Y=\sum_{i=1}^n S_i \quad \text{where}\quad S_i\coloneqq\frac{1}{n}(\mathbf{v}(x_i)\mathbf{v}(x_i)^T-\Sigma).\]
These $S_i$ are i.i.d. random matrices with mean zero. Furthermore, we have the uniform bound 
\[\|S_i\|=\frac{1}{n}\|\mathbf{v}(x_i)\mathbf{v}(x_i)^T-\Sigma\|\leq \frac{1}{n}(\|\mathbf{v}(x_i)\mathbf{v}(x_i)^T\|+\|\Sigma\|)\leq \frac{2R_{\mathbf{v}}^2}{n}.\]
Since $Y$ is Hermitian, the matrix variance $\nu(Y)$ is simply given by 
\[\|\E[Y^2]\|=\left\|\sum_{i=1}^n\E[S_i^2]\right\|.\]
For each summand, direct computation shows 
\[\E[S_i^2]\preceq \frac{R_{\mathbf{v}}^2\Sigma}{n^2},\]
where $\preceq$ indicates linear matrix inequality. Summing over the individual variances yields 
\[0\preceq \sum_{i=1}^n\E[S_i^2]\preceq\frac{R_{\mathbf{v}}^2\Sigma}{n},\]
and taking spectral norm leaves
\[\nu(Y)\leq \frac{R_{\mathbf{v}}^4}{n}.\]

Finally, we may invoke the matrix Bernstein inequality \cref{thm:bernstein}:
\[\E\|Y\|\leq\sqrt{\frac{2R_{\mathbf{v}}^4\log(2d+2)}{n}}+\frac{2R_{\mathbf{v}}^2\log(2d+2)}{3n}.\]
For large $n$ the second term is lower order, so we may conclude 
\[\E\|Y\|=O\left(R_\mathbf{v}^2\sqrt{\frac{\log d} n}\right).\]
Thus
\[\E\|\Sigma_n\Sigma^{-1}-I\|\leq \|\Sigma^{-1}\|\E\|\Sigma_n-\Sigma\|=O\left(\tau R_{\mathbf{v}}^2\sqrt{\frac{\log d}{n}}\right),\]
and 
\[\E_{\mathbf{s}^f}|\mathbf{a}^T(\Sigma_n\Sigma^{-1}-I)\mathbf{v}(x)|^2\leq \E_{\mathbf{s}^f}\|\mathbf{a}^T\|^2\|\Sigma_n\Sigma^{-1}-I\|^2\|\mathbf{v}(x)\|^2=O\left(\frac{\alpha^2 \tau^2 R_\mathbf{v}^{6}\log d}{n}\right).\] 

Putting everything together, we arrive at the desired bound:
\begin{align*}
    \RR(G_{\tilde{\theta}}) &= \E_{\mathbf{s}^f}|G_{\tilde{\theta}}(\mathbf{s}^f)-f(x)|^2 \\
    &\leq 2\E_{\mathbf{s}^f}|G_{\tilde{\theta}}(\mathbf{s}^f)-p_{f}(x)|^2+2\E_{x\sim \D_X, f\sim \D_{F}} |p_{f}(x)-f(x)|^2 \\
    &\leq 4\E_{\mathbf{s}^f}|\mathbf{a}^T(\Sigma_n\Sigma^{-1}-I)\mathbf{v}(x)|^2+4\E_{\mathbf{s}^f}|M|^2+2\E_{x\sim \D_X, f\sim \D_{F}} |p_{f}(x)-f(x)|^2 \\
    &= O\left(\frac{\alpha^2 \tau^2 R_\mathbf{v}^{6}\log d}{n}+(\delta R_\mathbf{v}^2\tau)^2+\delta^2\right),
\end{align*} 
where $O(\cdot)$ hides absolute constants.
\qed

\subsection{Proof of \cref{lem:staterror}}\label{sec:staterrorproof}
We bound the term
\[
    \E_\Gamma[\sup_{G_\theta\in \G}|\RR(G_\theta)-\RR^\Gamma(G_\theta)|]
\]
with standard techniques from statistical learning theory. For notational purposes, we define the loss function for a network $G_\theta$ at a sample $\mathbf{s}^f$ as
\[\L_\theta(\mathbf{s}^f)\coloneqq (G_\theta(\mathbf{s}^f)-f(x))^2.\]

Define the function class $\L\circ\G\coloneqq\{\mathbf{s}^f\mapsto \L_\theta(\mathbf{s}^f)\mid \theta\in \Theta\}$, and recall the empirical Rademacher complexity of $\L\circ\G$ over the training set $\Gamma$ is defined as  
\[R_\Gamma(\L\circ\G)\coloneqq \frac{1}{L}\E_{\boldsymbol{\sigma}\sim\{\pm1\}^L} \bigg[\sup_{\theta\in \Theta}\sum_{\ell=1}^L \sigma_\ell\L_\theta(\mathbf{s}^{f_\ell})\bigg].\]

By Lemma 26.2 in \citet{Shalev}, we have the following bound in terms of Rademacher complexity:
\[\E_\Gamma[\sup_{G_\theta\in \G}|\RR(G_\theta)-\RR^\Gamma(G_\theta)|]\leq 2\E_\Gamma[R_\Gamma(\L\circ\G)].\]
Note that the exact referenced theorem in the book does not include absolute value, but the above statement also holds due to Rademacher complexity being symmetric in sign. 

This Rademacher complexity term can be bounded using Dudley's integral theorem \cite{Dudley}, which yields
\[\E_\Gamma[R_\Gamma(\L\circ\G)]\leq \inf_{\alpha\geq 0} \bigg( 2\alpha+\frac{12}{\sqrt{L}}\int_\alpha^C\sqrt{\log \NN(\epsilon, \L\circ\G, \|\cdot\|_\infty)} \mathrm{d}\epsilon \bigg) \]
where $\NN(\epsilon, \L\circ \G, \|\cdot\|_\infty)$ denotes the $\epsilon$-covering number of $\L\circ \G$ under the $L^\infty$-norm. 

The covering number is computed in \cref{lem:coveringnum}, which gives \[\NN(\epsilon,\L\circ\G,\|\cdot\|_\infty) =O\left(\frac{2^{L_G^2} L_{\mathrm{FFN}}U^{3 L_T} d_{\mathrm{embed}}^{18 L_G^2 L_{\mathrm{FFN}}} \kappa^{6 L_G^2 L_{\mathrm{FFN}}+1} m_{G}^{L_G^2} \ell^{L_G^2}R}{\epsilon}\right)^{p_G},\]
where 
\[p_G=L_{G}\left(3 d_{\mathrm{embed}}^2 m_{G}+L_{\mathrm{FFN}} w_{\mathrm{FFN}}^2\right).\]

Denote the numerator by $C$. With this in hand, we may compute the above integral: setting $\alpha=\frac{1}{\sqrt{L}}$, we find
\[
    \E_\Gamma[\sup_{G_\theta\in \G}|\RR(G_\theta)-\RR^\Gamma(G_\theta)|] = O\left(\frac{\sqrt{p_G\log(C\sqrt{L})}}{\sqrt{L}}\right).
\]

Inserting the hyperparameters used for the construction in \cref{lem:approxerror} results in a bound of
\[\E_\Gamma[\sup_{G_\theta\in \G}|\RR(G_\theta)-\RR^\Gamma(G_\theta)|]  \\
    =O\bigg(\frac{\sqrt{nd^3(\log d)^3\log((d^4R_\mathbf{v}^2n^2U^2+\tau)L)}}{\sqrt{L}}\bigg).\]
This completes the proof. \qed

The covering number is computed in the following lemma, borrowing results from \citet{Havrilla24}.

\begin{lemma}\label{lem:coveringnum}
Let $\epsilon>0$. The transformer network class $\G(L_G,m_G,d_{\mathrm{embed}},\ell, L_{\mathrm{FFN}}, w_{\mathrm{FFN}},R,\kappa)$ with bounded input $\|H\|_\infty\leq U$ satisfies 
\[\NN(\epsilon,\L\circ\G,\|\cdot\|_\infty) =O\left(\frac{2^{L_G^2} L_{\mathrm{FFN}}U^{3 L_T} d_{\mathrm{embed}}^{18 L_G^2 L_{\mathrm{FFN}}} \kappa^{6 L_G^2 L_{\mathrm{FFN}}+1} m_{G}^{L_G^2} \ell^{L_G^2}R}{\epsilon}\right)^{p_G},\]
where 
\[p_G=L_{G}\left(3 d_{\mathrm{embed}}^2 m_{G}+L_{\mathrm{FFN}} w_{\mathrm{FFN}}^2\right).\]
\end{lemma}

\begin{proof}
Let $\eta>0$ and $G_\theta, G_{\theta'}\in\mathcal{G}$ be networks such that $\|\theta-\theta'\|_\infty < \eta$. Fix an input $\mathbf{s}$ which satisfies $\|\mathbf{s}\|_\infty\leq U$. We derive a bound on $|G_\theta(\mathbf{s})-G_{\theta'}(\mathbf{s})|$ in terms of $\eta$, which we can then control to ultimately derive a bound on the covering number.

Explicitly, the difference is given by 
\[|G_\theta(\mathbf{s})-G_{\theta'}(\mathbf{s})|=|\mathrm{D}\circ \mathrm{B}_{L_G}\circ\cdots\circ \mathrm{B}_1\circ (\mathrm{PE+E}(\mathbf{s}))-\mathrm{D}'\circ \mathrm{B}'_{L_G}\circ\cdots\circ \mathrm{B}'_1\circ (\mathrm{PE+E}'(\mathbf{s}))|.\]
Since the decoder is fixed to output a specific entry of the matrix, it is clear that this is at most
\[\| \mathrm{B}_{L_G}\circ\cdots\circ \mathrm{B}_1\circ (\mathrm{PE+E}(\mathbf{s}))-\mathrm{B}'_{L_G}\circ\cdots\circ \mathrm{B}'_1\circ (\mathrm{PE+E}'(\mathbf{s}))\|_{\max}.\]

Moreover, note that the encoding and positional encoding are also fixed: if we let $H=\mathrm{PE+E}(s)$, we have that $\|H\|_{\max}\leq \max(U,1)\leq U+1$. 

It remains to bound the difference between two transformer blocks. We utilize the results from \citet{Havrilla24}. First, the difference between multi-head attention layers (MHA) with ReLU activation on an input $H$ with $\|H\|_{\max} \leq \overline{U}$ is bounded by 
\[\|\mathrm{MHA}_1(H)-\mathrm{MHA}_2(H)\|^{\mathrm{ReLU}}_{\max}\leq 3\kappa^3d_{\mathrm{embed}}^6\overline{U}^3m_G\ell \eta.\]

Furthermore, their proof uses the fact that the Lipschitz constant of ReLU is bounded above by $1$ to remove it, which also holds for linear attention. The rest of the proof is independent of activation, ultimately yielding the same bound for linear attention and ReLU. 

Incorporating the feedforward layer, we have
\begin{align*}
     & \|\mathrm{FFN}_1(H+\mathrm{MHA}_1(H))-\mathrm{FFN}_2(H+\mathrm{MHA}_2(H))\|_{\max} \\
     &\leq 3 \kappa^{3+L_{\mathrm{FFN}}} w_{\mathrm{FFN}}^{2 L_{\mathrm{FFN}}} d_{\mathrm{embed}}^6 \overline{U}^3 m_{G} \ell \eta+L_{\mathrm{FFN}}\left(w_{\mathrm{FFN}}\left(2 d_{\mathrm{embed}}^6 \kappa^3 \overline{U} m_{G} \ell\right)+2\right)\left(\kappa w_{\mathrm{FFN}}\right)^{L_{\mathrm{FFN}}-1} \eta.
\end{align*}

Finally, with the residual connection we get the bound between transformer blocks: 
\begin{align*}
    \|\mathrm{B}_1(H)-\mathrm{B}_2(H)\|_{\max} &= \|(H+\mathrm{MHA}_1(H)+\mathrm{FFN}_1(H+\mathrm{MHA}_1(H))) \\ &\quad -(H+\mathrm{MHA}_2(H)+\mathrm{FFN}_2(H+\mathrm{MHA}_2(H)))\|_{\max} \\
    &\leq \bigg(4 \kappa^{3+L_{\mathrm{FFN}}} w_{\mathrm{FFN}}^{2 L_{\mathrm{FFN}}} d_{\mathrm{embed}}^6 \overline{U}^3 m_{G} \ell \\ &\quad +L_{\mathrm{FFN}}\left(w_{\mathrm{FFN}}\left(2 d_{\mathrm{embed}}^6 \kappa^3 \overline U m_{G} \ell\right)+2\right)\left(\kappa w_{\mathrm{FFN}}\right)^{L_{\mathrm{FFN}}-1}\bigg) \eta.
\end{align*}

A standard triangle inequality argument allows us to decompose the whole network into differences between individual components, whereupon we compute the total bound
\[|G_\theta(\mathbf{s})-G_{\theta'}(\mathbf{s})|=O\left( 2^{L_G^2+1} L_{\mathrm{FFN}} U^{3 L_T} d_{\mathrm{embed}}^{18 L_G^2 L_{\mathrm{FFN}}} \kappa^{6 L_G^2 L_{\mathrm{FFN}}} m_{G}^{L_G^2} \ell^{L_G^2} \eta\right).\]

Since the encoder and decoder are fixed, a network in $\G$ has a total of
\begin{align*} |\theta|=\sum_{i=1}^{L_G}\left|\theta_{\mathrm{B}_i}\right| &=L_{G}\left(3 d_{\mathrm{embed}}^2 m_{G}+L_{\mathrm{FFN}} w_{\mathrm{FFN}}^2\right)\end{align*}
parameters. We will refer to this quantity as $p_G$. Since each parameter is in $[-\kappa,\kappa]$, it takes at most $\frac{2\kappa}{\eta}$ points to cover that interval, which results in a covering number bound of 
\begin{align*} 
\NN(\epsilon,\G,\|\cdot\|_\infty) &\leq \left(\frac{2\kappa\cdot 2^{L_G^2+1} L_{\mathrm{FFN}} U^{3 L_T} d_{\mathrm{embed}}^{18 L_G^2 L_{\mathrm{FFN}}} \kappa^{6 L_G^2 L_{\mathrm{FFN}}} m_{G}^{L_G^2} \ell^{L_G^2}}{\epsilon}\right)^{p_G} \\ 
&= O\left(\frac{2^{L_G^2} L_{\mathrm{FFN}}U^{3 L_T} d_{\mathrm{embed}}^{18 L_G^2 L_{\mathrm{FFN}}} \kappa^{6 L_G^2 L_{\mathrm{FFN}}+1} m_{G}^{L_G^2} \ell^{L_G^2}}{\epsilon}\right)^{p_G}.
\end{align*}

Recall that we actually want the covering number after postcomposition with the loss function $\L$, which we can connect to the above quantity by the triangle inequality:
\[\|\L_\theta(\mathbf{s}^f)-\L_{\theta'}(\mathbf{s}^f)\|_\infty\leq 4R_G\|G_\theta(s)-G_{\theta'}(s)\|_\infty.\]
Therefore, we have
\[\NN(\epsilon,\L\circ\G,\|\cdot\|_\infty) =O\left(\frac{2^{L_G^2} L_{\mathrm{FFN}}U^{3 L_T} d_{\mathrm{embed}}^{18 L_G^2 L_{\mathrm{FFN}}} \kappa^{6 L_G^2 L_{\mathrm{FFN}}+1} m_{G}^{L_G^2} \ell^{L_G^2}R_G}{\epsilon}\right)^{p_G}\]
as desired.
\end{proof}

\section{Results for Splines}
\subsection{Assumptions for \cref{thm:splinegenerror}}\label{sec:splineassumptions}

\begin{assumption}[Linear Spline Approximability of Function Class]
For constants $\delta, \alpha, R_F\geq 0$, the distribution $\D_F$ is supported on the function class
\[
F(\delta,\alpha,R_F)\coloneqq  \bigg\{f \mid \min_{s\in\Pi^\alpha}\|f-s\|_{L^\infty(X)}\leq \delta \\ \text{ and }\|f\|_{L^\infty(X)}\leq R_F\bigg\}
\]
where $\Pi^\alpha$ consists of linear splines expressible as 
\[s(x)=\sum_{i=0}^m \alpha_iB_1\left(\frac{x-t_i}{h}\right)\]
with coefficients satisfying $\|(\alpha_0,\ldots,\alpha_m)\|\leq\alpha$.
\end{assumption}

\begin{assumption}
$\D_X$ is supported on $X\subseteq[-R,R]$ and the covariance matrix $\Sigma\coloneqq \E_{x\sim\D_X}[\mathbf{b}(x)\mathbf{b}(x)^T]$ has eigenvalues bounded away from $0$, such that 
\[ \|\Sigma^{-1}\|= \frac{1}{\lambda_{\min}(\Sigma)}\leq \tau_\mathrm{spline}\]
for some $\tau_\mathrm{spline}>0$.
\end{assumption}

\subsection{Proof of \cref{thm:splinegenerror}}\label{sec:splineproof}
We proceed in the same way as the proof of \cref{thm:generror}. The main difference is the construction used for the approximation error, which then suggests slightly different network hyperparameters used to determine the statistical error. We have the same error decomposition into
\[
  \E_\Gamma[\RR(G_{\hat\theta}^\Gamma)]\leq 2 \underbrace{\E_\Gamma[\sup_{G_\theta\in \G}|\RR(G_\theta)-\RR^\Gamma(G_\theta)|]}_{\text{I: Statistical Error}}
+ \underbrace{\RR\left(G_{\tilde{\theta}}\right).}_{\text{II: Approximation Error}}
\]

\textbf{Approximation Error}

Our input is embedded as
\[
    H=\mathrm{PE+E}(\mathbf{s}^f)=
    \begin{bmatrix}
    x_1 & x_2 & x_3 & \cdots & x_n & x \\
    0 & 0 & 0 & \cdots & 0 & 0 \\
    \vdots &  &  & \ddots & & \vdots \\
    0 & 0 & 0 & \cdots & 0 & 0 \\
    y_1 & y_2 & y_3 & \cdots & y_n & 0 \\
    0 & 0 & 0 & \cdots & 0 & 0 \\
    \mathcal{I}_1 & \mathcal{I}_2 & \mathcal{I}_3 & \cdots & \mathcal{I}_n & \mathcal{I}_{n+1} \\
    1 & 1 & 1 & \cdots & 1 & 1
    \end{bmatrix}\in \R^{(m+7)\times (n+1)}.
\]

In the degree $q=1$ case, the cardinal B-spline formula from \eqref{eq:cardinalbspline} simplifies to 
\[
    B_1(x)=\mathrm{ReLU}(x)-2\mathrm{ReLU}(x-1)+\mathrm{ReLU}(x-2),
\]
and the corresponding basis is
\[
    \left\{B_1^j(x)\coloneqq B_1\left(\frac{x-t_j}{h}\right) \mid j=0,1,\ldots, m\right\}
\]
where $t_0=a-h$. Since $h>0$, positive homogeneity of ReLU allows us to express each basis element as
\begin{equation}\label{eq:expandedlinearspline}
    B_1^j(x)=\frac{1}{h}\mathrm{ReLU}(x-t_j)-\frac{2}{h}\mathrm{ReLU}(x-t_j-h)+\frac{1}{h}\mathrm{ReLU}(x-t_j-2h).
\end{equation}

As we did before, we featurize each $x_i$ separately (again, for indexing convenience we may refer to $x$ as $x_{n+1}$). Since we use the same features across inputs, the computations are repeated in parallel via MHA. The goal of this step is to build in-context the following feature submatrix:
\begin{equation}\label{eq:1dsplinefeatures}
    \begin{bmatrix}
    B_1^0(x_1) & B_1^0(x_2) & B_1^0(x_3) & \cdots & B_1^0(x_n) & B_1^0(x) \\
    B_1^1(x_1) & B_1^1(x_2) & B_1^1(x_3) & \cdots & B_1^1(x_n) & B_1^1(x) \\
    \vdots & \vdots & \vdots & \ddots & \vdots & \vdots \\
    B_1^m(x_1) & B_1^m(x_2) & B_1^m(x_3) & \cdots & B_1^m(x_n) & B_1^m(x) \\
    \end{bmatrix}.
\end{equation}

Since we are using a fixed set of knots, the values involving $h$ and $t_j$ are essentially constant and can be encoded in the parameters explicitly. While we could break this computation up into simple binary arithmetic operations \citep{Shen}, the Interaction Lemma interestingly allows for the computation of the degree $1$ B-spline values in \emph{one step} from this point (and with only one wide attention block), which we demonstrate here.

For $1\leq i\leq n+1$ and $0\leq j\leq m$, each of the three summands in \eqref{eq:expandedlinearspline} can be realized exactly with one ReLU attention head. Since we use residual attention, we can automatically add the three results by having them placed in the same entry, and we can do these in parallel. We will put the value of $B_1^j(x_i)$ in the $(j+2,i)$-index entry.

For the first term $\frac{1}{h}\mathrm{ReLU}(x_i-t_j)$, we let $V_{i,j,1}=\frac{1}{h}\mathbf{e}_{j+2}\mathbf{e}_{d_{\mathrm{embed}}}^T$ and pick data kernels 

\[
Q^{\mathrm{data}}_{i,j,1} = 
\left[\begin{array}{ccccc|ccc} 
    0 & & & & & 0 & 0 & 1 \\ 
      & 0 & & & & 0 & 0 & 0 \\
      & & \ddots & & & \vdots & \vdots & \vdots \\
      & & & 0 & & 0 & 0 & 0 \\
      & & & & 0 & 0 & 0 & 1
\end{array}\right] 
\quad\quad\quad 
K^{\mathrm{data}}_{i,j,1} = 
\left[\begin{array}{cccccc|ccc} 
    1 & & & & & &  0 & 0 & 0 \\ 
    & 0 & & & & & \vdots & \vdots & \vdots \\
	& & \ddots & & & & 0 & 0 & 0 \\ 
    & &  & 0 & & & \vdots & \vdots & \vdots\\
    & & & & 0 & & 0 & 0 & 0 \\
    & &  &  & & 0 & 0 & 0 & -t_j
\end{array}\right],
\]

both in $\R^{(d_\mathrm{embed}-3)\times d_\mathrm{embed}}$. Recall the bottom row of $H$ is constant 1. By \cref{lem:interaction}, there exists a ReLU attention head $A_{i,j,1}$ such that on columns,
\[
    A_{i,j,1}(\mathbf{h}_i) = \sigma(\langle Q^{\mathrm{data}}_{i,j,1}\mathbf{h}_i, K^{\mathrm{data}}_{i,j,1}\mathbf{h}_i\rangle) V_i\mathbf{h}_i = \frac{1}{h}\sigma(x_i-t_j)\mathbf{e}_{j+2} 
\]
and $A_i(\mathbf{h}_l)=0$ when $l\neq i$.

The other two terms are similar: for $-\frac{2}{h}\mathrm{ReLU}(x-t_j-h)$, we let $V_{i,j,2}=-\frac{2}{h}\mathbf{e}_{j+2}\mathbf{e}_{d_{\mathrm{embed}}}^T$ with data kernels 
\[
Q^{\mathrm{data}}_{i,j,2} = 
\left[\begin{array}{ccccc|ccc} 
    0 & & & & & 0 & 0 & 1 \\ 
      & 0 & & & & 0 & 0 & 0 \\
      & & \ddots & & & \vdots & \vdots & \vdots \\
      & & & 0 & & 0 & 0 & 0 \\
      & & & & 0 & 0 & 0 & 1
\end{array}\right] 
\quad\quad\quad 
K^{\mathrm{data}}_{i,j,2} = 
\left[\begin{array}{cccccc|ccc} 
    1 & & & & & &  0 & 0 & 0 \\ 
    & 0 & & & & & \vdots & \vdots & \vdots \\
	& & \ddots & & & & 0 & 0 & 0 \\ 
    & &  & 0 & & & \vdots & \vdots & \vdots\\
    & & & & 0 & & 0 & 0 & 0 \\
    & &  &  & & 0 & 0 & 0 & -t_j-h
\end{array}\right],
\]

and for $\frac{1}{h}\mathrm{ReLU}(x-t_j-2h)$, we take $V_{i,j,3}=\frac{1}{h}\mathbf{e}_{j+2}\mathbf{e}_{d_{\mathrm{embed}}}^T$ and data kernels 

\[
Q^{\mathrm{data}}_{i,j,3} = 
\left[\begin{array}{ccccc|ccc} 
    0 & & & & & 0 & 0 & 1 \\ 
      & 0 & & & & 0 & 0 & 0 \\
      & & \ddots & & & \vdots & \vdots & \vdots \\
      & & & 0 & & 0 & 0 & 0 \\
      & & & & 0 & 0 & 0 & 1
\end{array}\right]  
\quad\quad\quad 
K^{\mathrm{data}}_{i,j,3} = 
\left[\begin{array}{cccccc|ccc} 
    1 & & & & & &  0 & 0 & 0 \\ 
    & 0 & & & & & \vdots & \vdots & \vdots \\
	& & \ddots & & & & 0 & 0 & 0 \\ 
    & &  & 0 & & & \vdots & \vdots & \vdots\\
    & & & & 0 & & 0 & 0 & 0 \\
    & &  &  & & 0 & 0 & 0 & -t_j-2h
\end{array}\right].
\]
\cref{lem:interaction} allows us to find corresponding ReLU attention heads such that residual MHA over all three heads returns

\[
    \sum_{k=1}^3 A_{i,j,k}(H)+H = \begin{bmatrix}
    x_1 & \cdots & x_{i-1} & x_i & x_{i+1} & \cdots & x_n & x \\
    0 & \cdots & 0 & 0 & 0 & \cdots & 0 & 0 \\
    \vdots & \ddots & \vdots & \vdots & \vdots & \ddots & \vdots & \vdots \\
    0 & \cdots & 0 & 0 & 0 & \cdots & 0 & 0 \\
    0 & \cdots & 0 & B_1^j(x_i) & 0 & \cdots & 0 & 0 \\
    0 & \cdots & 0 & 0 & 0 & \cdots & 0 & 0 \\
    \vdots & \ddots & \vdots & \vdots & \vdots & \ddots & \vdots & \vdots \\
    0 & \cdots & 0 & 0 & 0 & \cdots & 0 & 0 \\
    y_1 & \cdots & y_{i-1} & y_i & y_{i+1} & \cdots & y_n & 0 \\
    0 & \cdots & 0 & 0 & 0 & \cdots & 0 & 0 \\
    \mathcal{I}_1 & \cdots & \mathcal{I}_{i-1} & \mathcal{I}_i & \mathcal{I}_{i+1} & \cdots & \mathcal{I}_n & \mathcal{I}_{n+1} \\
    1 & \cdots & 1 & 1 & 1 & \cdots & 1 & 1
\end{bmatrix}.
\]

Proceeding similarly, with a total of $3(n+1)(m+1)=O(nm)$ heads, the residual MHA exactly outputs 
\begin{equation}\label{1dsplinefeatures}\mathrm{MHA}(H)+H=\sum_{\substack{1\leq i\leq n+1 \\ 0\leq j\leq m \\ 1\leq k \leq 3}}A_{i,j,k}(H)+H=
    \begin{bmatrix}
    x_1 & x_2 & x_3 & \cdots & x_n & x \\
    B_1^0(x_1) & B_1^0(x_2) & B_1^0(x_3) & \cdots & B_1^0(x_n) & B_1^0(x) \\
    B_1^1(x_1) & B_1^1(x_2) & B_1^1(x_3) & \cdots & B_1^1(x_n) & B_1^1(x) \\
    \vdots & \vdots & \vdots & \ddots & \vdots & \vdots \\
    B_1^m(x_1) & B_1^m(x_2) & B_1^m(x_3) & \cdots & B_1^m(x_n) & B_1^m(x) \\
    y_1 & y_2 & y_3 & \cdots & y_n & 0 \\
    0 & 0 & 0 & \cdots & 0 & 0 \\
    \mathcal{I}_1 & \mathcal{I}_2 & \mathcal{I}_3 & \cdots & \mathcal{I}_n & \mathcal{I}_{n+1} \\
    1 & 1 & 1 & \cdots & 1 & 1
    \end{bmatrix}.
\end{equation}

By \cref{lem:interaction}, we have the weight bound $\|\theta\|_\infty = O(m^4n^2R^2)$. Conveniently, we do not actually need the feedforward component here, since the ReLU is used as a part of the value computation. In contrast, when constructing monomials (and isolating other simple arithmetic operations), we add a large positive shift such that ReLU acts as identity on our result, and we unshift using the FFN that follows via \cref{lem:decrementing}.

Recall we denote the feature vectors as $\mathbf{b}(x_i)\coloneqq [B_1^0(x_i), B_1^1(x_i), \ldots, B_1^m(x_i)]^T$, the empirical uncentered covariance as
$\Sigma_n\coloneqq \frac{1}{n}\sum_{i=1}^n\mathbf{b}(x_i)\mathbf{b}(x_i)^T$, and the population form as $\Sigma\coloneqq \E[\mathbf{b}(x_i)\mathbf{b}(x_i)^T]$.

The last linear attention block uses the parameters
\[V=\begin{bmatrix}
    0_{(m+2)\times (m+2)} & 0_{(m+2)\times 1} & 0_{(m+2)\times 4} \\
    0_{1\times (m+2)} & 1 & 0_{1\times 4} \\
    0_{4\times (m+2)} & 0_{4\times 1} & 0_{4\times 4}
\end{bmatrix}\qquad \text{and} \qquad Q=\begin{bmatrix}
    0_{1\times 1} & 0_{1\times (m+1)} & 0_{1\times 5} \\
    0_{(m+1)\times 1} & \Sigma^{-1} & 0_{(m+1)\times 5} \\
    0_{5\times 1} & 0_{5\times (m+1)} & 0_{5\times 5}
\end{bmatrix},\]
both in $\R^{d_\mathrm{embed}\times d_\mathrm{embed}}$. The parameters are similar to those from \eqref{eq:monomiallinearparams}.

The output after a final decoder which reads off the entry at index $(m+3,n+1)$ is given by 
\[\left(\frac{1}{n}\sum_{i=1}^ny_i\mathbf{b}(x_i)^T\right)\Sigma^{-1}\mathbf{b}(x).\]

The rest of the proof proceeds identically to \cref{sec:approxerrorproof}, again using the matrix Bernstein inequality \cref{thm:bernstein} to quantify the empirical to population gap. Notice $\|B_1^i(x)\|_{L^\infty(X)}=1$ for all $i$, so we have the $\infty$-norm upper bound on feature vectors $R_\mathbf{b}=1$.

Substituting everything in, we arrive at the approximation error bound
\[
    \RR(G_{\tilde{\theta}}) = O\left(\frac{\alpha^2\tau^2\log m}{n}+\delta ^2\tau^2+\delta^2\right).
\]

\textbf{Statistical Error}

The proof in \cref{sec:staterrorproof} applies here too: the only architectural difference is the embedding, which does not affect the covering number computation. The new hyperparameters determined by the construction above are $d_\mathrm{embed}=m+7$, $L_G=2$, $m_G=O(nm)$, $L_\mathrm{FFN}=0$, $w_\mathrm{FFN}=0$, and $\kappa=O(m^4n^2R^2)$. Substituting these in, we get a statistical error bound of 
\[
    \E_\Gamma[\sup_{G_\theta\in \G}|\RR(G_\theta)-\RR^\Gamma(G_\theta)|]=O\left(\frac{\sqrt{nm^3\log((m^4R^2n^2U^2+\tau_{\mathrm{spline}})L)}}{\sqrt{L}})\right).
\]

Adding the two errors yields the desired bound. \qed

\subsection{Higher Degree Splines}\label{sec:higherordersplines}

Incorporating power operations we used for the polynomial feature representation, we can easily extend the linear spline construction to higher degree cardinal B-splines. We outline the quadratic case, where the basis functions $B_2^j(x)$ can be expressed by the formula

\[
B_2\left(\frac{x-t_j}{h}\right) = \frac{1}{2h^2}\bigg[(\mathrm{ReLU}(x-t_j))^2-3(\mathrm{ReLU}(x-t_j-h))^2+3(\mathrm{ReLU}(x-t_j-2h))^2-(\mathrm{ReLU}(x-t_j-3h))^2\bigg].
\]

We require a larger $d_\mathrm{embed}$ to store intermediate computations. In particular, assuming we are still using $m$ knots (which corresponds to $m+2$ features for degree $q=2$ splines), we will require $1+4(m+2)$ additional rows from the linear spline case to store each of $\gamma^j_k(x)\coloneqq\mathrm{ReLU}(x-t_j-kh)$ for $k=0,1,2,3$ and $-1\leq j\leq m$, since squaring must be done in a separate transformer block. Thus the embedding matrix is
\[  
H=\mathrm{PE+E}(\mathbf{s}^f)=
\begin{bmatrix}
x_1 & x_2 & x_3 & \cdots & x_n & x \\
0 & 0 & 0 & \cdots & 0 & 0 \\
\vdots &  &  & \ddots & & \vdots \\
0 & 0 & 0 & \cdots & 0 & 0 \\
y_1 & y_2 & y_3 & \cdots & y_n & 0 \\
0 & 0 & 0 & \cdots & 0 & 0 \\
\mathcal{I}_1 & \mathcal{I}_2 & \mathcal{I}_3 & \cdots & \mathcal{I}_n & \mathcal{I}_{n+1} \\
1 & 1 & 1 & \cdots & 1 & 1
\end{bmatrix}\in\R^{(5m+16)\times(n+1)}.
\]

The first transformer block will construct $\gamma^j_k(x_i)$ for each $0\leq k\leq 3$, $-1\leq j\leq m$, and $1\leq i\leq n+1$ via attention heads $A_{i,j,k}$. Let $V_{i,j,k}=\mathbf{e}_{4(j+1)+k+2}\mathbf{e}_{d_\mathrm{embed}}^T$ and choose data kernels of the form
\[
Q^{\mathrm{data}}_{i,j,k} = 
\left[\begin{array}{ccccc|ccc} 
    0 & & & & & 0 & 0 & 1 \\ 
      & 0 & & & & 0 & 0 & 0 \\
      & & \ddots & & & \vdots & \vdots & \vdots \\
      & & & 0 & & 0 & 0 & 0 \\
      & & & & 0 & 0 & 0 & 1
\end{array}\right] 
\quad\quad\quad 
K^{\mathrm{data}}_{i,j,k} = 
\left[\begin{array}{cccccc|ccc} 
    1 & & & & & &  0 & 0 & 0 \\ 
    & 0 & & & & & \vdots & \vdots & \vdots \\
	& & \ddots & & & & 0 & 0 & 0 \\ 
    & &  & 0 & & & \vdots & \vdots & \vdots\\
    & & & & 0 & & 0 & 0 & 0 \\
    & &  &  & & 0 & 0 & 0 & -t_j-kh
\end{array}\right] 
.\]

\cref{lem:interaction} implies the existence of attention heads $A_{i,j,k}$ such that the residual MHA yields
\[
\mathrm{MHA}(H)+H=\sum_{\substack{1\leq i\leq n+1 \\ -1\leq j\leq m \\ 0\leq k \leq 3}}A_{i,j,k}(H)+H=
    \begin{bmatrix}
    x_1 & x_2 & x_3 & \cdots & x_n & x \\
    \gamma_0^{-1}(x_1) & \gamma_0^{-1}(x_2) & \gamma_0^{-1}(x_3) & \cdots & \gamma_0^{-1}(x_n) & \gamma_0^{-1}(x) \\
    \gamma_1^{-1}(x_1) & \gamma_1^{-1}(x_2) & \gamma_1^{-1}(x_3) & \cdots & \gamma_1^{-1}(x_n) & \gamma_1^{-1}(x) \\
    \gamma_2^{-1}(x_1) & \gamma_2^{-1}(x_2) & \gamma_2^{-1}(x_3) & \cdots & \gamma_2^{-1}(x_n) & \gamma_2^{-1}(x) \\
    \gamma_3^{-1}(x_1) & \gamma_3^{-1}(x_2) & \gamma_3^{-1}(x_3) & \cdots & \gamma_3^{-1}(x_n) & \gamma_3^{-1}(x) \\
    \gamma_0^0(x_1) & \gamma_0^0(x_2) & \gamma_0^0(x_3) & \cdots & \gamma_0^0(x_n) & \gamma_0^0(x) \\
    \vdots & \vdots & \vdots & \ddots & \vdots & \vdots \\
    \gamma_3^m(x_1) & \gamma_3^m(x_2) & \gamma_3^m(x_3) & \cdots & \gamma_3^m(x_n) & \gamma_3^m(x) \\
    0 & 0 & 0 & \cdots & 0 & 0 \\
    \vdots & \vdots & \vdots & \ddots & \vdots & \vdots \\
    0 & 0 & 0 & \cdots & 0 & 0 \\
    y_1 & y_2 & y_3 & \cdots & y_n & 0 \\
    0 & 0 & 0 & \cdots & 0 & 0 \\
    \mathcal{I}_1 & \mathcal{I}_2 & \mathcal{I}_3 & \cdots & \mathcal{I}_n & \mathcal{I}_{n+1} \\
    1 & 1 & 1 & \cdots & 1 & 1
    \end{bmatrix}.
\]

The next transformer block constructs the final features: it squares the previous results, multiplies them by their respective constant coefficients, and adds the output to the appropriate entry. Such attention heads are a variation of the ones used for multiplication in the monomial feature representation. Again, let us define each attention head $A_{i,j,k}$ for $0\leq k\leq 3$, $-1\leq j\leq m$, and $1\leq i\leq n+1$ via attention heads $A_{i,j,k}$. Let $V_{i,j,k}=\mathbf{e}_{4(m+1)+j+7}\mathbf{e}_{d_\mathrm{embed}}^T$ and specify data kernels

\begin{equation*}
Q^{\mathrm{data}}_{i,j,k} = 
\left[\begin{array}{cccccccc|ccc} 
    0 & & & & & & & &  0 & 0 & 0 \\ 
    & \ddots & & & & & & & \vdots & \vdots & \vdots \\
    & & 0 & & & & & & 0 & 0 & 0 \\ 
    & & & 1 & & & & & 0 & 0 & 0 \\
	& & & & 0 & & & & 0 & 0 & 0 \\ 
    & & & & & \ddots & & & \vdots & \vdots & \vdots\\
    & & & & & & 0 & & 0 & 0 & 0 \\
    & & & & & & & 0 & 0 & 0 & 1
\end{array}\right] 
\;
K^{\mathrm{data}}_{i,j,k} = 
\left[\begin{array}{cccccccc|ccc} 
    0 & & & & & & & &  0 & 0 & 0 \\ 
    & \ddots & & & & & & & \vdots & \vdots & \vdots \\
    & & 0 & & & & & & 0 & 0 & 0 \\ 
    & & & \frac{\upsilon_k}{2h^2} & & & & & 0 & 0 & 0 \\
	& & & & 0 & & & & 0 & 0 & 0 \\ 
    & & & & & \ddots & & & \vdots & \vdots & \vdots\\
    & & & & & & 0 & & 0 & 0 & 0 \\
    & & & & & & & 0 & 0 & 0 & M
\end{array}\right]
\end{equation*}
where $\upsilon_0=1, \upsilon_1=-3, \upsilon_2=3,$ and $\upsilon_3=-1$. Both $Q^{\mathrm{data}}_{i,j,k}$ and $K^{\mathrm{data}}_{i,j,k}$ have the left side nonzero entry at index $((4(j+1)+k+2),(4(j+1)+k+2))$. By \cref{lem:interaction}, we have attention heads such that after subtracting off the $+4M$ constants with the FFN via \cref{lem:decrementing}, the output of the transformer block is 

\[
\mathrm{MHA}(H)+H=\sum_{\substack{1\leq i\leq n+1 \\ -1\leq j\leq m \\ 0\leq k \leq 3}}A_{i,j,k}(H)+H=
    \begin{bmatrix}
    x_1 & x_2 & x_3 & \cdots & x_n & x \\
    \gamma_0^{-1}(x_1) & \gamma_0^{-1}(x_2) & \gamma_0^{-1}(x_3) & \cdots & \gamma_0^{-1}(x_n) & \gamma_0^{-1}(x) \\
    \vdots & \vdots & \vdots & \ddots & \vdots & \vdots \\
    \gamma_3^m(x_1) & \gamma_3^m(x_2) & \gamma_3^m(x_3) & \cdots & \gamma_3^m(x_n) & \gamma_3^m(x) \\
    B_2^{-1}(x_1) & B_2^{-1}(x_2) & B_2^{-1}(x_3) & \cdots & B_2^{-1}(x_n) & B_2^{-1}(x) \\
    B_2^0(x_1) & B_2^0(x_2) & B_2^0(x_3) & \cdots & B_2^0(x_n) & B_2^0(x) \\
    \vdots & \vdots & \vdots & \ddots & \vdots & \vdots \\
    B_2^m(x_1) & B_2^m(x_2) & B_2^m(x_3) & \cdots & B_2^m(x_n) & B_2^m(x) \\
    y_1 & y_2 & y_3 & \cdots & y_n & 0 \\
    0 & 0 & 0 & \cdots & 0 & 0 \\
    \mathcal{I}_1 & \mathcal{I}_2 & \mathcal{I}_3 & \cdots & \mathcal{I}_n & \mathcal{I}_{n+1} \\
    1 & 1 & 1 & \cdots & 1 & 1
    \end{bmatrix}.
\]

Adjusting the linear attention parameters with appropriate zero padding solves the least square problem with the final linear attention layer.

\section{In-Context Polynomial Regression on Vector-Valued Functions}\label{sec:vectorvalued}

Suppose we follow the same setup established in \cref{sec:setup} except we now regress on functions $f\colon \R\to \R^D$. Each function can be written as $f=(f^1,f^2,\ldots, f^D)$, for which we can regress in parallel. Our samples are now of the form $\mathbf{s}^f=(x_1, \mathbf{y}_1, x_2, \mathbf{y}_2, \ldots, x_n, \mathbf{y}_n;x)$, where $\mathbf{y}_i=(y_i^1,y_i^2,\ldots,y_i^D)=f(x_i)$.

We now require \cref{ass:polyapproximability} for each $f^j$ from $j=1$ to $D$ instead. In other words, any vector-valued function we regress has each coordinate approximable by degree $\leq d$ polynomials with bounded coefficients. 

\begin{assumption}[Coordinate-wise Polynomial Approximability of Function Class]
For a given degree $d$ and constant $\delta, \alpha, R_F\geq 0$, the distribution $\D_F$ is supported on the function class
\[F(\delta,\alpha,d,R_F)\coloneqq  \bigg\{f=(f^1,f^2,\ldots, f^D) \mid \min_{p\in\Pi_d^\alpha}\|f^j-p\|_{L^\infty(X)}\leq \delta \\ \text{ and }\|f^j\|_{L^\infty(X)}\leq R_F \text{ for } 1\leq j\leq D\bigg\}\]
where $\Pi_d^\alpha$ consists of polynomials of the form $p(x)=\sum_{i=0}^da_ix^i$ with coefficients satisfying $\|(a_0,\ldots,a_d)\|\leq\alpha$.
\end{assumption}

We also assume \cref{ass:conditioning}.

\begin{thm}[Generalization Error for In-Context Polynomial Regression of Vector-Valued Functions]
Fix hyperparameters for the network class $\G$ as 
\begin{align*}
    & L_G = O(\log d), m_G=O(dn), d_{\mathrm{embed}}=D+d+6, \ell=n+1 \\
    & L_{\mathrm{FFN}}=O(1), w_{\mathrm{FFN}}=D+d+6, \kappa = O(d^4R_\mathbf{v}^2n^2U^2+\tau).
\end{align*}
The minimizer of empirical loss $G^\Gamma_{\hat{\theta}}$ in this class satisfies 
\[
\E_\Gamma[\RR(G_{\hat\theta}^\Gamma)] = O\left(\frac{D\alpha^2 R_\mathbf{v}^{10}\log d}{n}+D\delta R_\mathbf{v}\tau+D\delta^2+\frac{\sqrt{nd(D+d)^2(\log d)^3\log((d^4R_\mathbf{v}^2n^2U^2+\tau)L)}}{\sqrt{L}}\right)
\]
\end{thm}

\begin{proof}
For the approximation, our construction works similarly if in embedding we replace the $y$-row in our embedding matrix with the corresponding $y$-submatrix 
\[\begin{bmatrix}
    \mathbf{y}_1 & \mathbf{y}_2 & \cdots & \mathbf{y}_n & \mathbf{0}
\end{bmatrix}=\begin{bmatrix}
    y_1^1 & \cdots & y_n^1 & 0 \\
    \vdots & \ddots & \vdots & \vdots \\
    y_1^n & \cdots & y_n^n & 0
\end{bmatrix}\in \R^{D\times (n+1)}.\]

The ReLU transformer blocks do not depend on those rows; we still use these to form the matrix 
\[\begin{bmatrix}
    \mathbf{v}_1 & \cdots & \mathbf{v}_n & \mathbf{v} \\
    \mathbf{y}_1 & \cdots & \mathbf{y}_n & 0 \\
    \mathbf{u}_1 & \cdots & \mathbf{u}_n & \mathbf{u}_{n+1}
\end{bmatrix}\in \R^{(D+d+6)\times (n+1)}\]
with the earlier blocks. Since this part is independent of the function outputs, this construction still requires the same number of transformer blocks and attention heads. 

For the final block, a simple adaptation of the linear attention parameters to

\[V=\begin{bmatrix}
    0_{(d+2)\times (d+2)} & 0_{(d+2)\times D} & 0_{(d+2)\times 4} \\
    0_{D\times (d+2)} & \tilde{V} & 0_{D\times 4} \\
    0_{4\times (d+2)} & 0_{4\times D} & 0_{4\times 4}
\end{bmatrix}\qquad \text{and} \qquad Q=\begin{bmatrix}
    0_{1\times 1} & 0_{1\times (d+1)} & 0_{1\times (D+4)} \\
    0_{(d+1)\times 1} & \tilde{Q} & 0_{(d+1)\times (D+4)} \\
    0_{(D+4)\times 1} & 0_{(D+4)\times (d+1)} & 0_{(D+4)\times (D+4)}
\end{bmatrix}\]
where $\tilde{V}=I_D$ is the identity matrix of size $D$ and $\tilde{Q}=\Sigma^{-1}\in\R^{(d+1)\times (d+1)}$ yields the predicted regression result with each $y^i$ in the $(i+d+1, n+1)$-th entry, whereupon our decoder now reads off that subcolumn. 

Recall the approximation error in the one-dimensional case is 
\[O\left(\frac{\alpha^2 R_\mathbf{v}^{10}\log d}{n}+\delta R_\mathbf{v}\tau+\delta^2\right).\] 
Our approximation error under the $\ell^2$-norm on $\R^D$ is thus
\[
    \RR(G_{\tilde{\theta}})\coloneqq \E_{\mathbf{s}^f}\|G_{{\tilde{\theta}}}(\mathbf{s}^f)-f(x)\|_2^2=O\left(D\left(\frac{\alpha^2 R_\mathbf{v}^{10}\log d}{n}+\delta R_\mathbf{v}\tau+\delta^2\right)\right).
\]

The statistical error also changes slightly: the only difference in network parameters is a slightly larger embedding dimension of $D+d+6$ (as opposed to the previous $d+7$). Including dependence on $D$ in the asymptotic bound, the statistical error is bounded by 
\[\E_\Gamma[\sup_{G_\theta\in \G}|\RR(G_\theta)-\RR^\Gamma(G_\theta)|]  
    =O\left(\frac{\sqrt{nd(D+d)^2(\log d)^3\log((d^4R_\mathbf{v}^2n^2U^2+\tau)L)}}{\sqrt{L}}\right).\]
Adding this to the aforementioned approximation error gives the total generalization error bound. 
\end{proof}

\section{Additional Numerical Results}\label{sec:additionalnumerical}

Unless otherwise specified, our experiments are run using the default setup and hyperparameters summarized in \cref{tab:hyperparams}.

\begin{table}[h]
\centering
\caption{Hyperparameters and Configurations}
\renewcommand{\arraystretch}{1.2}
\begin{tabular}{l p{8cm}}
\toprule
\textbf{Parameter} & \textbf{Value} \\
\midrule
\multicolumn{2}{l}{\textbf{Training}} \\
Optimizer & Adam \\
Learning Rate & $0.001$ \\
Batch Size & $512$ \\
Epochs & $50$ \\
Gradient Clipping & Max norm $= 1.0$ \\
Loss Function & MSE \\
\midrule
\multicolumn{2}{l}{\textbf{Experimental Setup}} \\
Polynomial Degree ($d$) & $4$ \\
Input Distribution ($\D_X$) & U$(-1,1)$ \\
Test Set Size & 1000 \\
\midrule
\multicolumn{2}{l}{\textbf{Model Architecture}} \\
Embedding Dim ($d_{\text{embed}}$) & $d + 7=11$ \\
Transformer Depth/Number of Blocks ($L_G$)\textsuperscript{1} & $d=4$ \\
FFN Depth ($L_\mathrm{FFN}$)\textsuperscript{2} & 2 \\
FFN Width ($w_\mathrm{FFN}$) & $d_{\mathrm{embed}}$ \\
Weight Initialization (Q, K, V) & $\mathcal{N}(0, 0.01)$ \\
Scaling (ReLU/Softmax) & $1 / \sqrt{d_{\text{embed}}} = 1 / \sqrt{11}$ \\
Scaling (Linear) & $1 / (\ell - 1)=1/n$ \\
\bottomrule
\multicolumn{2}{l}{\footnotesize \textsuperscript{1}Linear $\to$ ReLU $\to$ Linear.} \\
\multicolumn{2}{l}{\footnotesize \textsuperscript{2}Theory model uses $(L_G-1)$ ReLU blocks (w/ FFN) then 1 linear block (no FFN, 1 head).}
\end{tabular}
\label{tab:hyperparams}
\end{table}

Here we also present some additional experimental results, first continuing to investigate the effect of the number of heads. The results for a fixed $4$ attention heads per block are shown in \cref{fig:4headsscaling}. Our constructive theory also suggests a quasi-equivalence between width and depth; correspondingly, we try a deep narrow model ($16$ blocks with 1 head each), with results shown in \cref{fig:1heads16blocksscaling}. We also try a smaller model, using just a single head per attention layer with the standard $d=4$ blocks, and the results are shown in \cref{fig:1headsscaling}. We observe worse scaling in both $n$ and $L$ for the single-head case, most noticeable in the linear model. The starting absolute error did not change much, since the problem is likely simple enough even for the smaller capacity. 

We also ablate the FFN component, studying an attention-only transformer (also with $4$ heads per block). The results in \cref{fig:noffnscaling} show a slight increase in absolute error but still perform and scale surprisingly well, supporting our argument that approximation power can be found in attention itself. This also highlights the potential of depth, following previous work which shows that a single linear attention layer cannot outperform linear predictors on nonlinear tasks \citep{Sun}.

Finally, we test our framework with splines to regress piecewise linear functions, with shallow attention-only networks. The results are shown in \cref{fig:splinescaling}. In terms of scaling, softmax and ReLU perform similarly, while the purely linear model is noticeably worse, failing to achieve our asymptotic bounds. This confirms and extends the results of \citep{Sun}: linear attention alone, even when composed, is insufficient to regress nonlinear functions, but we find that introducing a nonlinearity makes attention-only networks capable of in-context nonlinear regression.

\begin{figure}[h]
    \centering
    \begin{subfigure}[b]{0.49\linewidth}
        \centering
        \includegraphics[width=\linewidth]{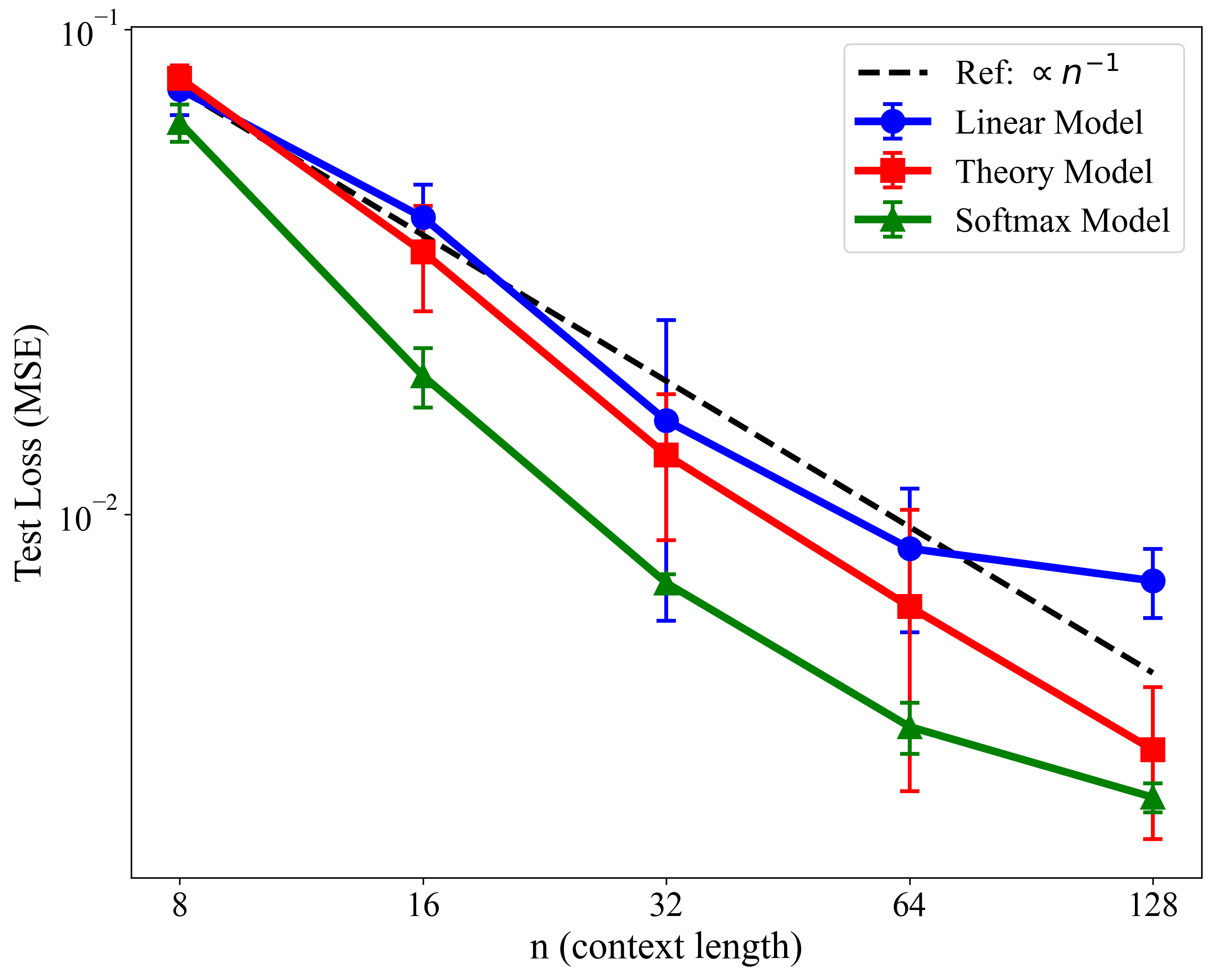}
        \caption{Test loss vs. context length ($n$) with fixed $L=32000$}
        \label{fig:4headsscalingn}
    \end{subfigure}
    \hfill
    \begin{subfigure}[b]{0.49\linewidth}
        \centering
        \includegraphics[width=\linewidth]{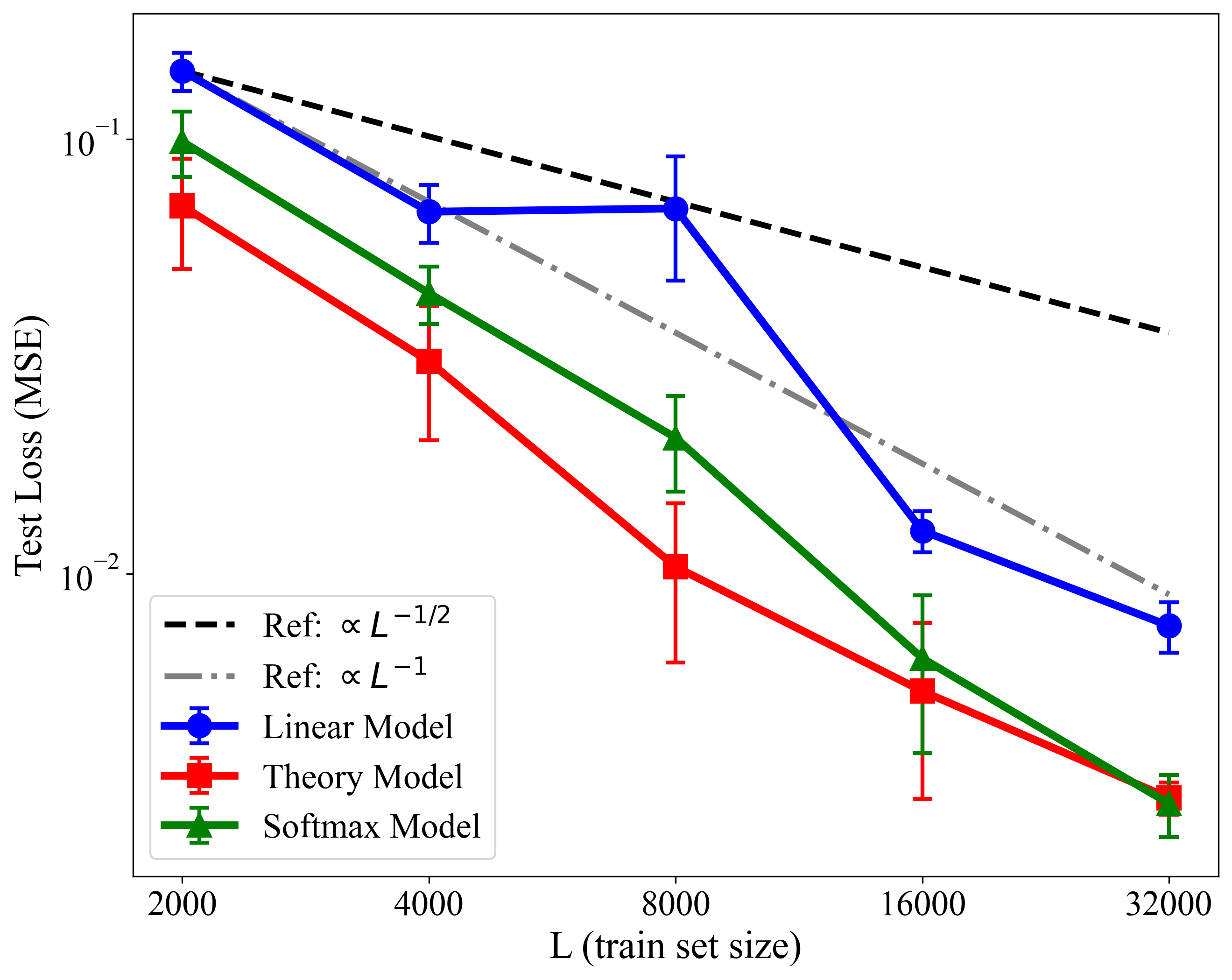}
        \caption{Test loss vs. training set size ($L$) with fixed $n=128$}
        \label{fig:4headsscalingL}
    \end{subfigure}
    \caption{Scaling results on polynomial regression tasks using $4$ attention heads per block (except for the final linear block in the theory model). }
    \label{fig:4headsscaling}
\end{figure}

\begin{figure}[H]
    \centering
    \begin{subfigure}[b]{0.49\linewidth}
        \centering
        \includegraphics[width=\linewidth]{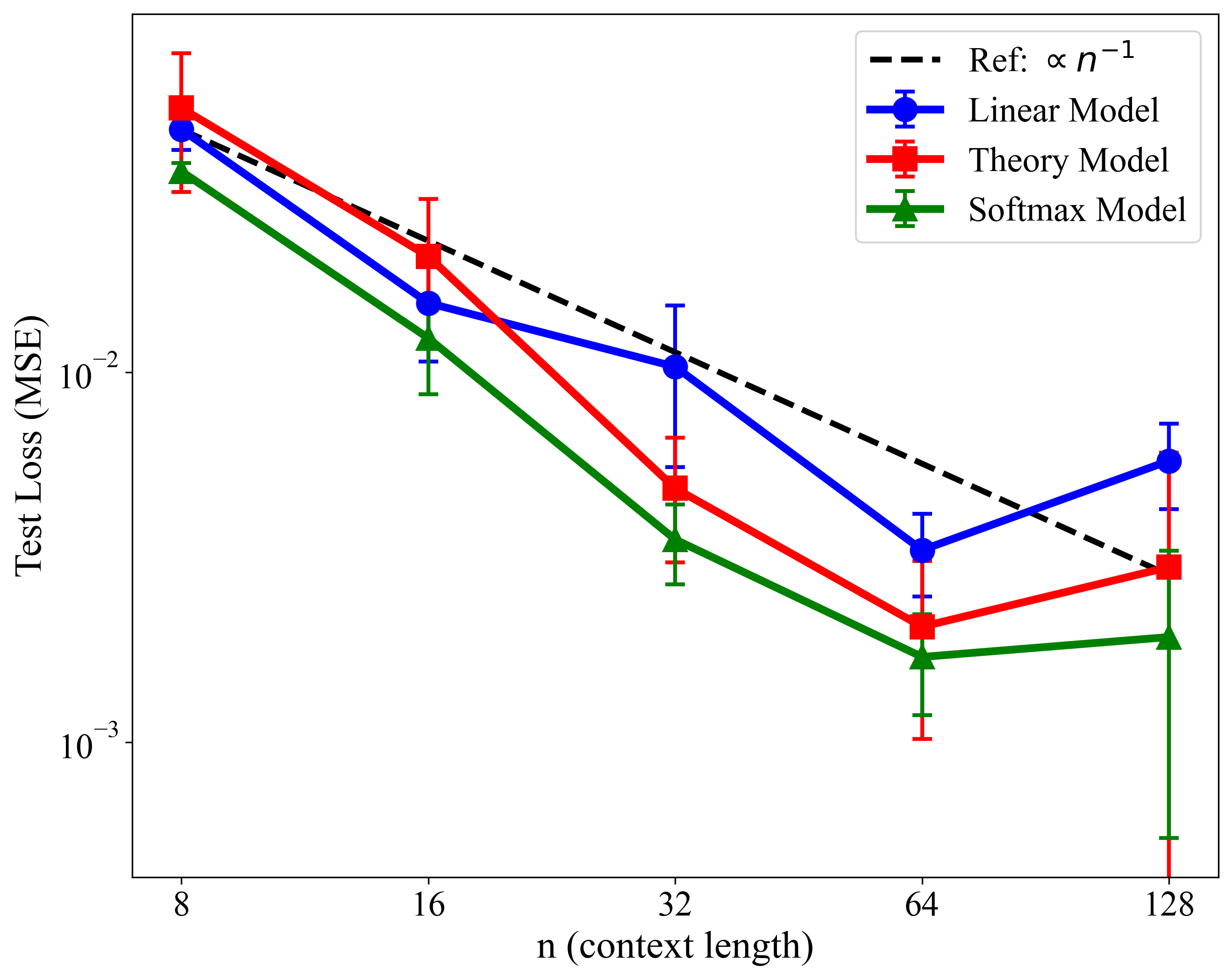}
        \caption{Test loss vs. context length ($n$) with fixed $L=32000$}
    \end{subfigure}
    \hfill
    \begin{subfigure}[b]{0.49\linewidth}
        \centering
        \includegraphics[width=\linewidth]{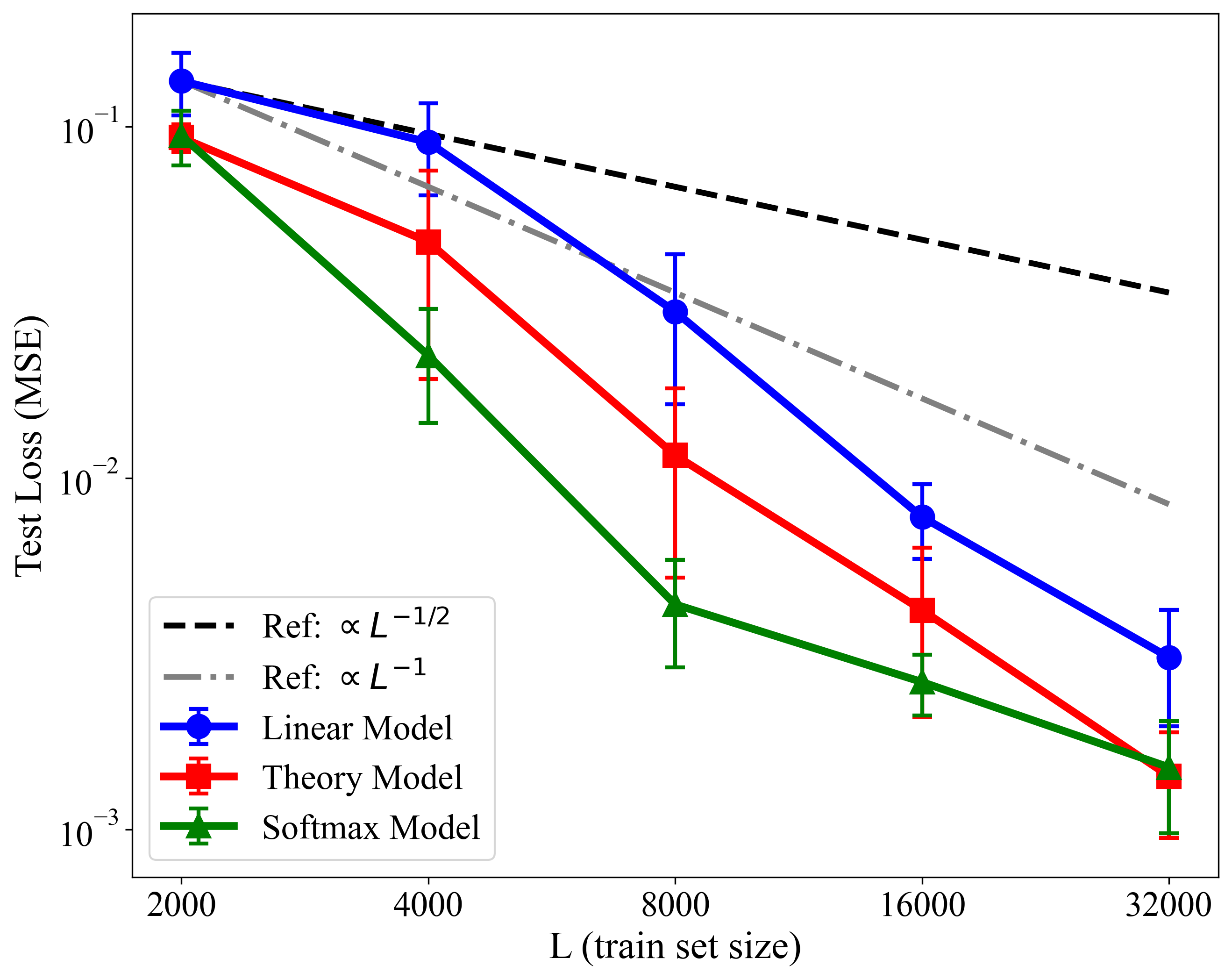}
        \caption{Test loss vs. training set size ($L$) with fixed $n=128$}
    \end{subfigure}
    \caption{Scaling results on polynomial regression tasks using $16$ blocks with $1$ attention head each.}
    \label{fig:1heads16blocksscaling}
\end{figure}

\begin{figure}[h]
    \centering
    \begin{subfigure}[b]{0.49\linewidth}
        \centering
        \includegraphics[width=\linewidth]{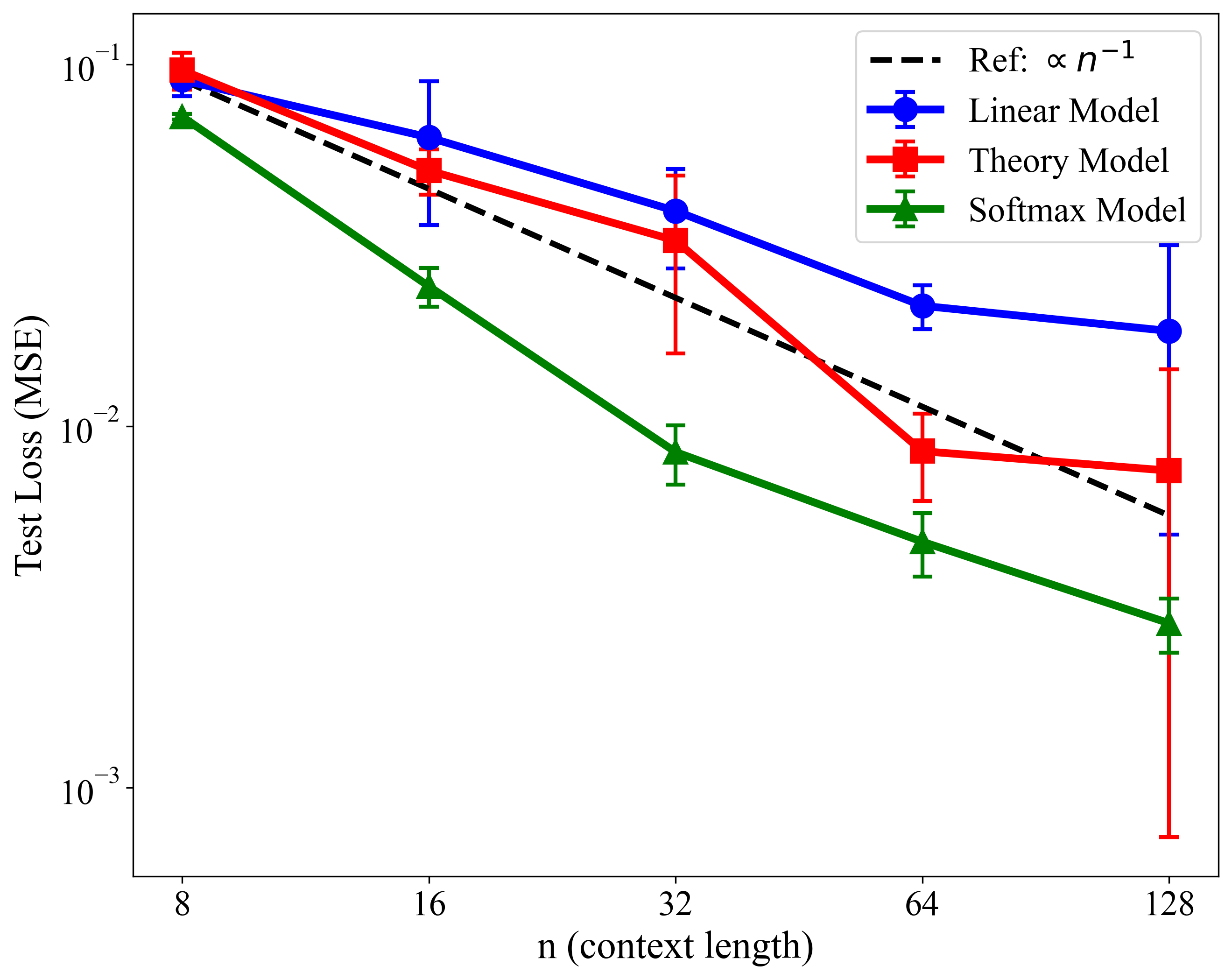}
        \caption{Test loss vs. context length ($n$) with fixed $L=32000$}
        \label{fig:1headsscalingn}
    \end{subfigure}
    \hfill
    \begin{subfigure}[b]{0.49\linewidth}
        \centering
        \includegraphics[width=\linewidth]{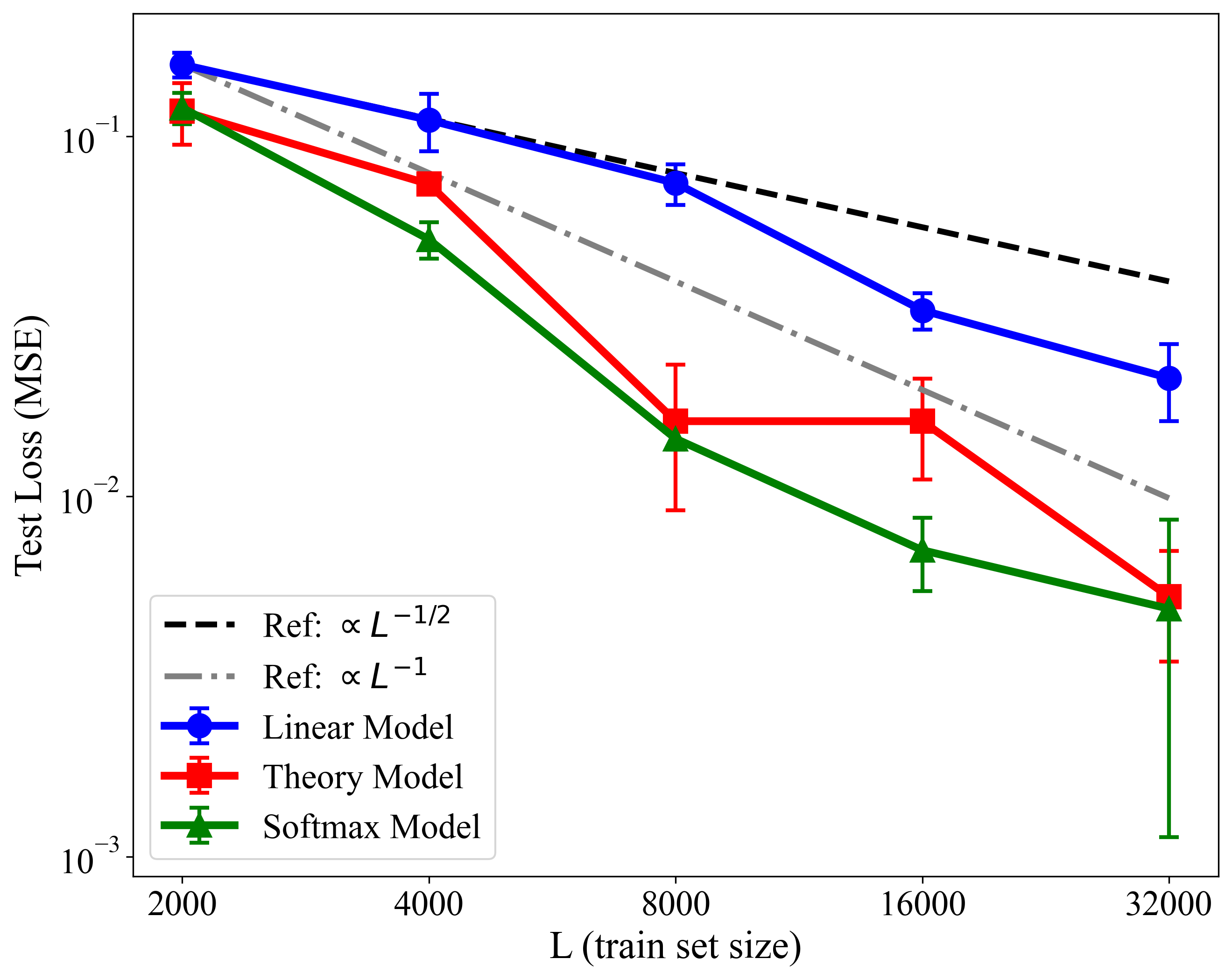}
        \caption{Test loss vs. training set size ($L$) with fixed $n=128$}
        \label{fig:1headsscalingL}
    \end{subfigure}
    \caption{Scaling results on polynomial regression tasks using a single attention head per block.}
    \label{fig:1headsscaling}
\end{figure}

\begin{figure}[h]
    \centering
    \begin{subfigure}[b]{0.49\linewidth}
        \centering
        \includegraphics[width=\linewidth]{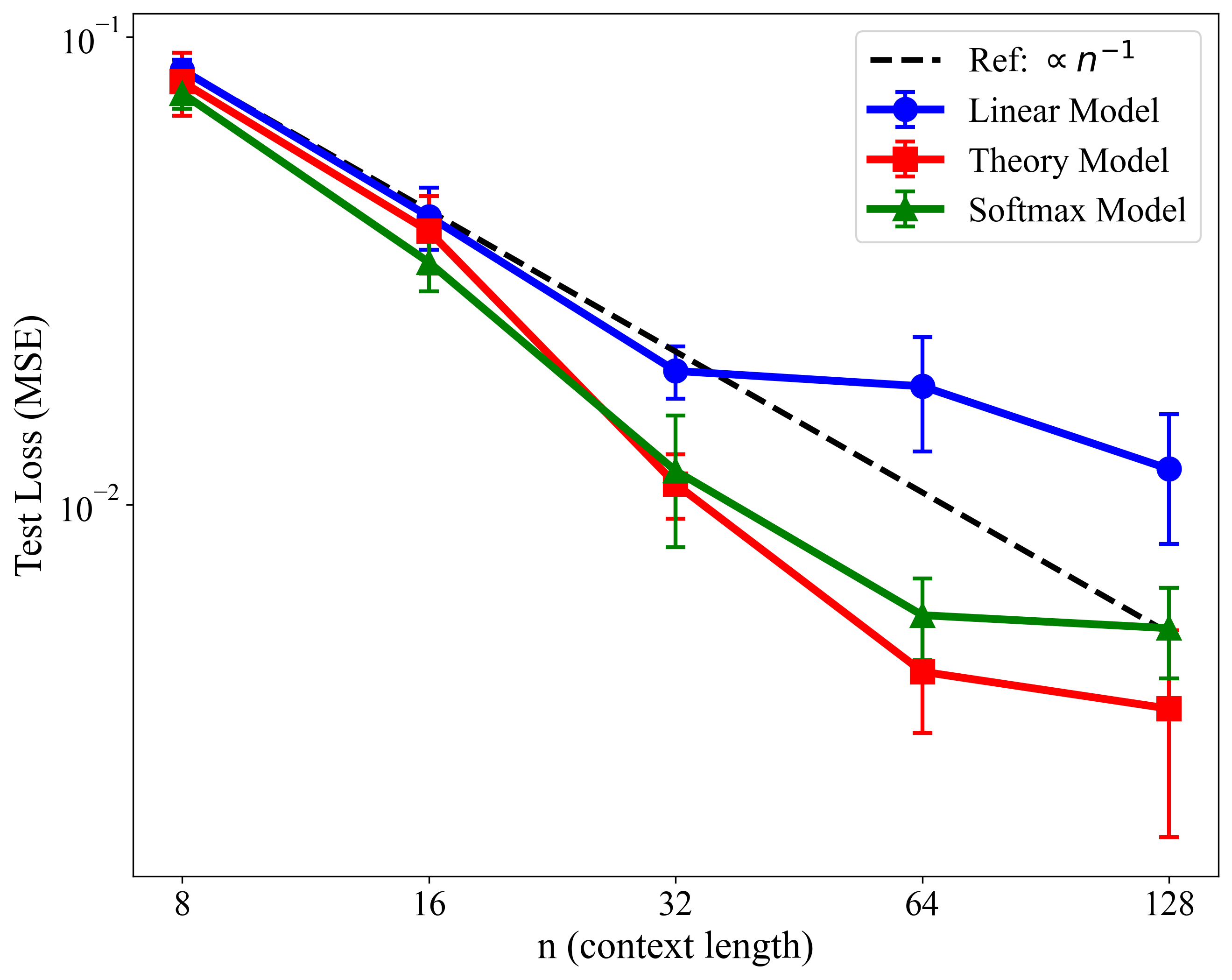}
        \caption{Test loss vs. context length ($n$) with fixed $L=32000$}
        \label{fig:noffnscalingn}
    \end{subfigure}
    \hfill
    \begin{subfigure}[b]{0.49\linewidth}
        \centering
        \includegraphics[width=\linewidth]{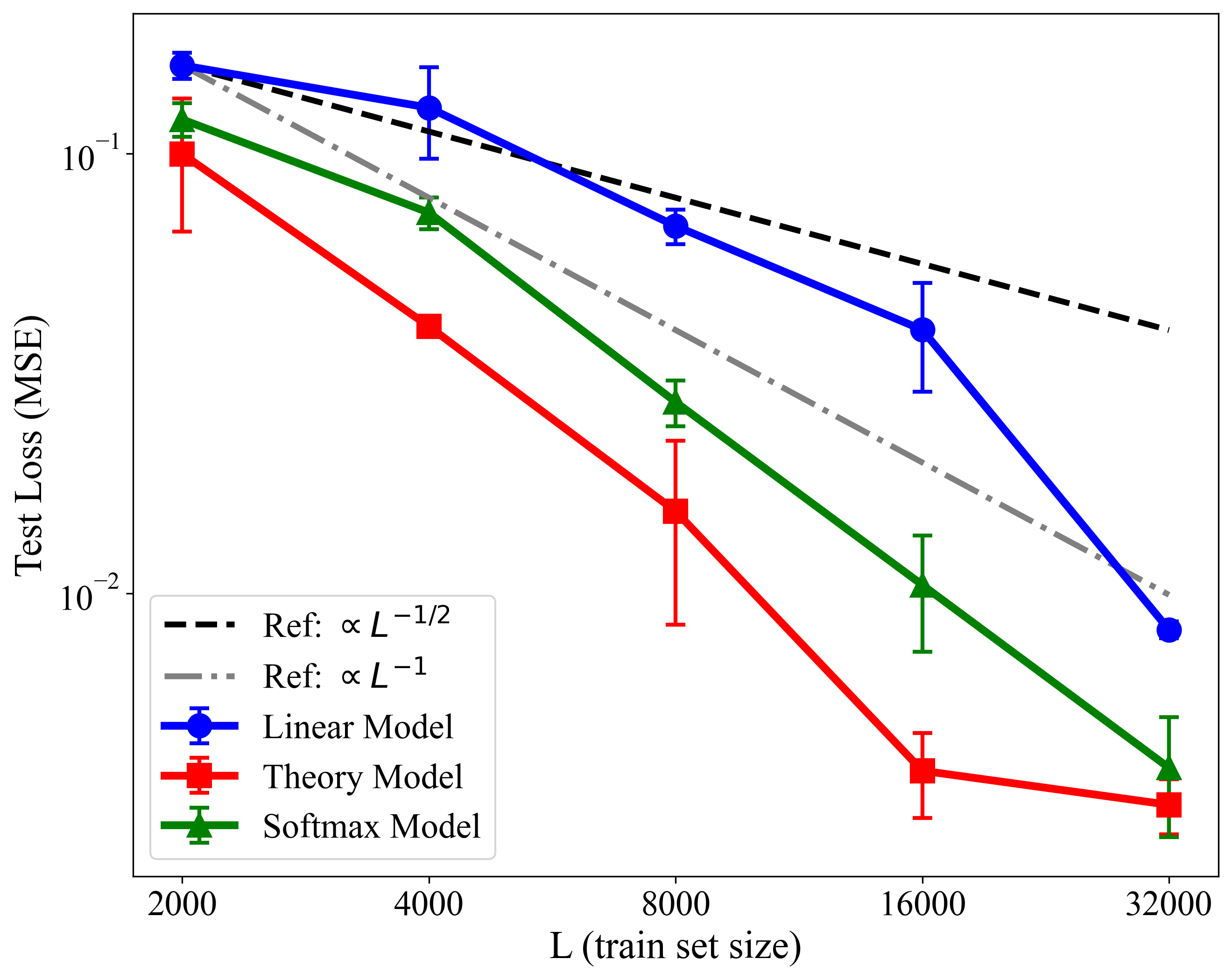}
        \caption{Test loss vs. training set size ($L$) with fixed $n=128$}
        \label{fig:noffnscalingL}
    \end{subfigure}
    \caption{Scaling results on polynomial regression tasks with no feedforward component, i.e. an attention-only transformer.}
    \label{fig:noffnscaling}
\end{figure}

\begin{figure}[h]
    \centering
    \begin{subfigure}[b]{0.49\linewidth}
        \centering
        \includegraphics[width=\linewidth]{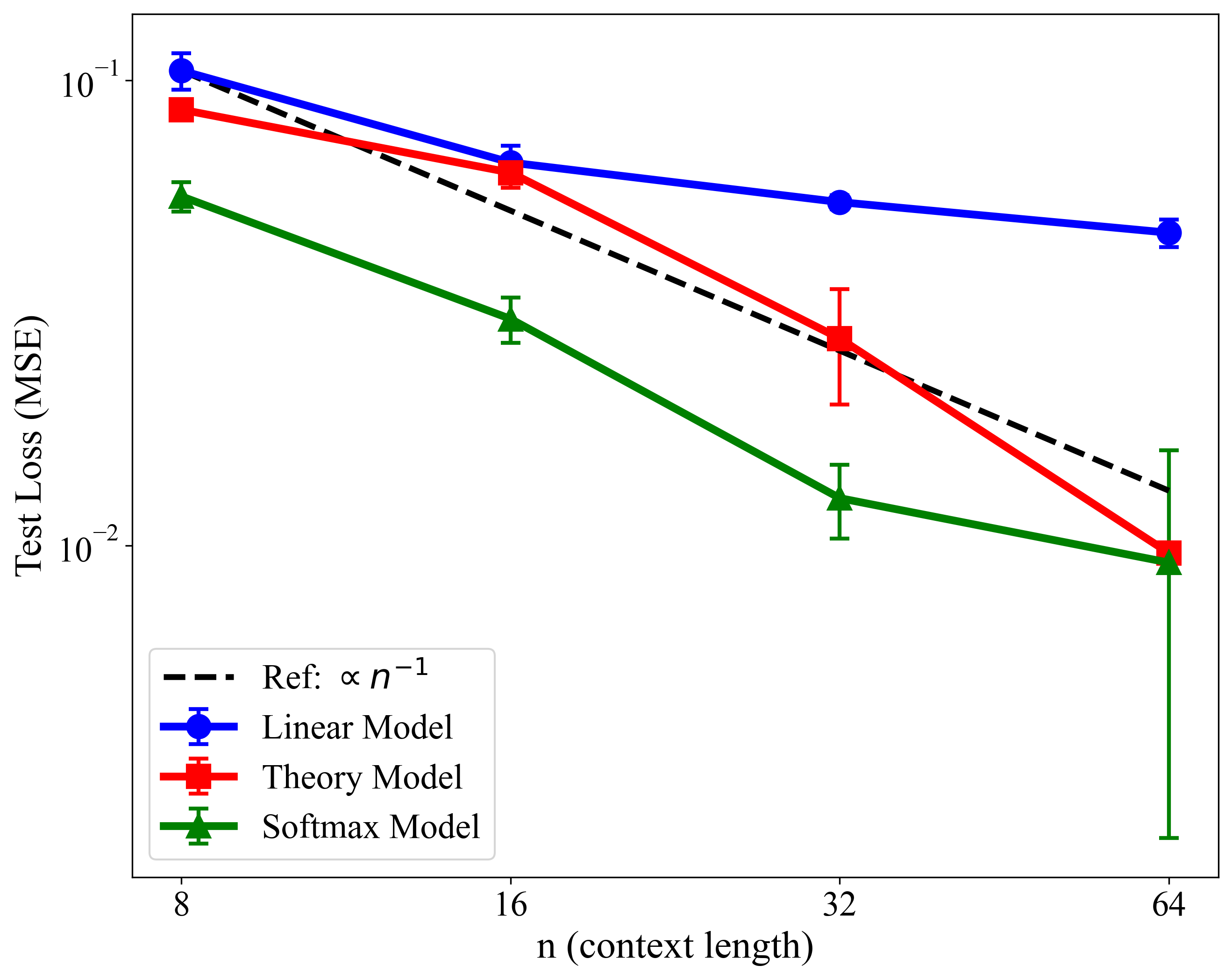}
        \caption{Test loss vs. context length ($n$) with fixed $L=16000$}
    \end{subfigure}
    \hfill
    \begin{subfigure}[b]{0.49\linewidth}
        \centering
        \includegraphics[width=\linewidth]{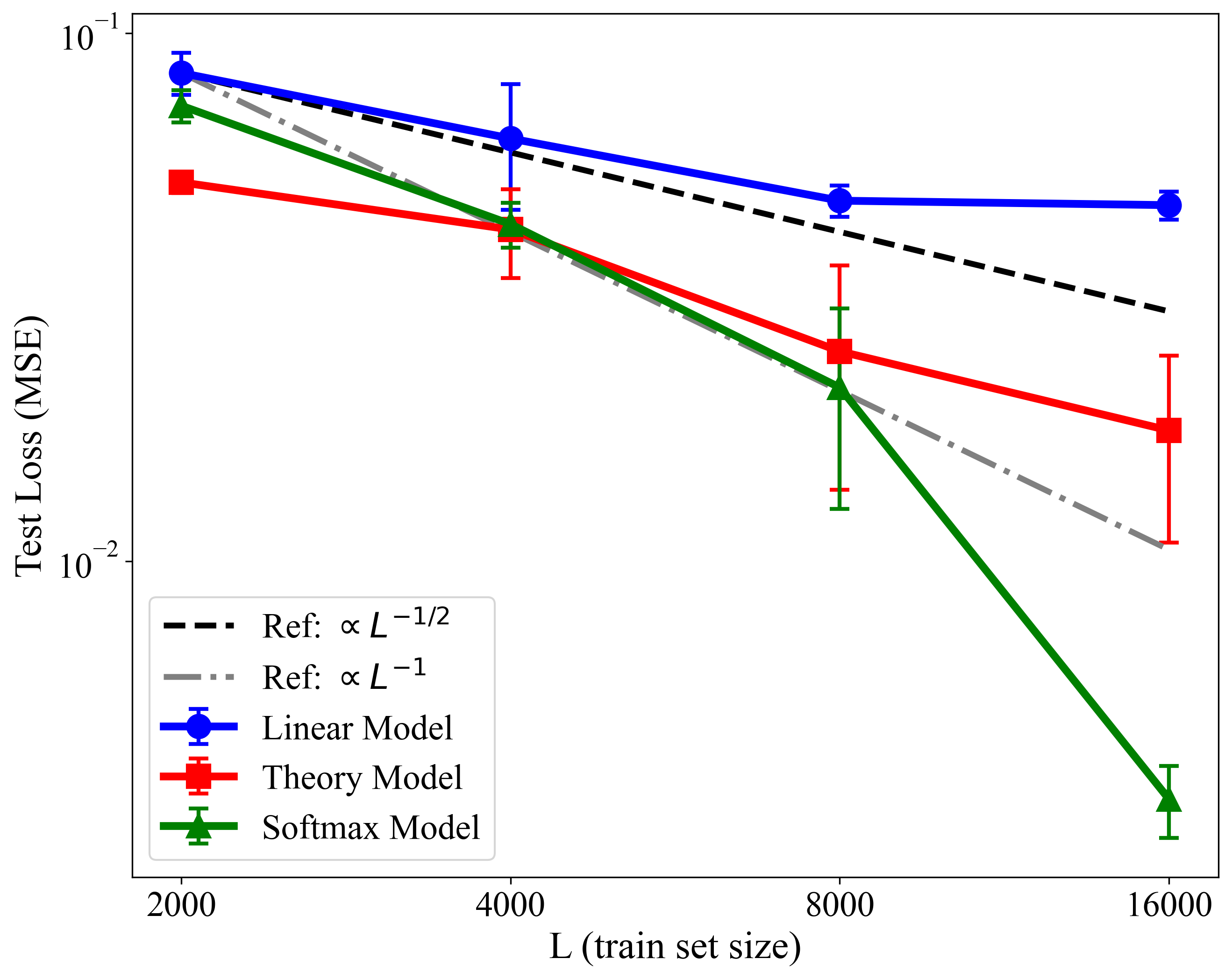}
        \caption{Test loss vs. training set size ($L$) with fixed $n=64$}
    \end{subfigure}
    \caption{Scaling results for regression of piecewise linear functions (linear splines), with domain split by $5$ equally spaced knots. All models used two blocks and no feedforward components. We used $n/8$ attention heads per block, except the theory model which uses a one-headed linear attention block after an $n/8$-headed ReLU attention block.}
    \label{fig:splinescaling}
\end{figure}

\end{document}